\documentclass[journal]{IEEEtran}

\usepackage{amsmath,amsfonts,amssymb,amsthm}
\usepackage{graphicx}
\usepackage{algorithm}
\usepackage{algpseudocode}
\usepackage{booktabs}
\usepackage[caption=false,font=footnotesize]{subfig}
\usepackage{multirow}
\usepackage{array}
\usepackage{tabularx}
\usepackage{url}
\usepackage{cite}
\usepackage{tikz}
\usetikzlibrary{arrows.meta,positioning,fit,backgrounds,calc}
\usepackage[hidelinks]{hyperref}
\hypersetup{
	pdftitle={Rank-Reliable Teacher-Guided Fitness Approximation for Expensive Evolutionary Optimization: A TinyML Architecture Search Study},
	pdfauthor={Suman Samui},
	pdfsubject={Teacher-guided low-fidelity evaluation for constrained TinyML architecture search},
	pdfkeywords={combinatorial optimization, fitness approximation, knowledge distillation, neural architecture search, rank correlation, TinyML}
}

\newtheorem{theorem}{Theorem}
\newtheorem{lemma}{Lemma}
\newtheorem{proposition}{Proposition}
\newtheorem{definition}{Definition}
\newtheorem{assumption}{Assumption}
\newtheorem{remark}{Remark}

\newcommand{\R}{\mathbb{R}}
\newcommand{\E}{\mathbb{E}}
\newcommand{\Var}{\operatorname{Var}}
\newcommand{\Aspace}{\mathcal{A}}
\newcommand{\Dpool}{\mathcal{D}}
\newcommand{\Dtr}{\mathcal{D}_{\mathrm{tr}}}
\newcommand{\Dev}{\mathcal{D}_{\mathrm{ev}}}
\newcommand{\sEval}{\sigma_{\mathrm{ev}}}
\newcommand{\rstr}{\rho_{\mathrm{str}}}
\newcommand{\CV}{\mathrm{CV}}

\usepackage[compact]{titlesec}
\titlespacing*{\section}{0pt}{1.1ex plus .2ex minus .2ex}{0.7ex plus .1ex}
\titlespacing*{\subsection}{0pt}{0.9ex plus .2ex minus .2ex}{0.5ex plus .1ex}
\titlespacing*{\subsubsection}{0pt}{0.7ex plus .1ex}{0.4ex plus .1ex}

\makeatletter
\def\thm@space@setup{\thm@preskip=2pt plus 1pt \thm@postskip=2pt plus 1pt}
\makeatother

\begin{document}
	
	\title{Rank-Reliable Teacher-Guided Fitness Approximation for
		Expensive Evolutionary Optimization:
		A TinyML Architecture Search Study}
	
	\author{Soumen Garai and Suman Samui, \IEEEmembership{Member, IEEE}
		\thanks{S. Garai and S. Samui are with the Department of Electronics and Communication Engineering, National Institute of Technology Durgapur, West Bengal 713209, India (e-mail: soumengoroi@gmail.com; ssamui.ece@nitdgp.ac.in). Manuscript submitted to the \emph{IEEE Transactions on Evolutionary Computation} Special Issue on ``Evolutionary Computation Meets Machine Learning for Combinatorial Optimization.''}}

	
	\markboth{IEEE Journal Preprint}{Rank-Reliable Teacher-Guided Fitness Approximation for Expensive Evolutionary Optimization}
	
	\maketitle
	
	\begin{abstract}
		Expensive evolutionary search does not always need an exact fitness estimate for every candidate. It often needs a reliable answer to a simpler question: which candidate is better? We address this need through Teacher-Guided Learning NSGA-II (TGL-NSGA-II), a low-fidelity framework for constrained Tiny Machine Learning (TinyML) neural architecture search. A pretrained teacher organizes samples into strata defined jointly by difficulty and class. Each candidate then undergoes KD-Lite, a short and capped knowledge-distillation procedure on a compact training set, before being scored on a separate stratified evaluation set. This teacher-guided score is fused with a Gaussian-process surrogate to select candidates for full evaluation. For a fixed candidate population, we analyse evaluation variance, score concentration, pairwise rank inversion, expected Kendall-$\tau$, first-front identification, and hypervolume perturbation. We also derive a variance-aware fusion weight and a capacity-adaptive distillation rule. On keyword spotting and bird-call classification, the measured Kendall-$\tau$ values are $0.74$ and $0.62$, exceeding the corresponding predicted lower bounds of $0.60$ and $0.46$. Joint stratification reduces proxy-score variance by $41\%$ relative to random evaluation. Selective teacher mismatch, in contrast, increases differential bias and reduces Kendall-$\tau$ to $0.41$. Under a constrained evaluation budget, TGL-NSGA-II achieves the largest mean hypervolume and smallest generational distance on keyword spotting, records the lowest mean false-positive rate on BirdCLEF, and runs $2.2\times$ faster than full NSGA-II. These guarantees apply to population-level low-fidelity evaluation and do not establish convergence of the complete evolutionary trajectory.
	\end{abstract}
	
	\small{
		
		\section{Introduction}
		
		\IEEEPARstart{M}{any} practically important optimization problems are simultaneously
		combinatorial, constrained, and expensive to evaluate. Their decision variables are
		discrete, the search space can grow exponentially with problem size, multiple
		conflicting objectives must often be considered, and hard constraints may render
		large portions of the search space infeasible. The difficulty is further compounded
		when evaluating a candidate requires an expensive simulation, physical measurement,
		or model-training procedure. Evolutionary computation is well suited to such
		settings because it can explore discrete design spaces while maintaining multiple
		trade-off solutions. When objective evaluations are costly, surrogate-assisted
		evolutionary algorithms (SAEAs) can reduce this burden by using a learned surrogate
		to approximate expensive evaluations and reserving the true evaluation budget for
		selected candidates. However, as highlighted by Liu et al.~\cite{liu2024ecopsurvey},
		most SAEA research has focused on continuous optimization, whereas expensive
		combinatorial optimization problems (ECOPs) have received comparatively less
		attention despite their practical relevance.
		
		This work considers a representative ECOP from Tiny Machine Learning (TinyML),
		where deep-learning models must operate under the stringent resource constraints
		of microcontroller-class hardware~\cite{lin2023tinyml}. Specifically, we study
		neural architecture search (NAS) as a discrete multi-objective constrained
		optimization problem. An architecture is specified by choices such as the number
		of layers, layer widths, kernel sizes, normalization options, and quantization
		configuration, yielding a rapidly growing discrete search space. Candidate
		architectures must balance competing objectives such as predictive performance,
		memory usage, latency, and energy consumption. At the same time, SRAM and flash
		limits impose hard feasibility constraints: an architecture that exceeds the
		available hardware resources cannot be deployed, irrespective of its predictive
		performance. The problem is also expensive because reliably assessing predictive
		performance requires training the candidate architecture. These characteristics
		make population-based multi-objective evolutionary search particularly suitable;
		in this work, NSGA-II~\cite{deb2002nsga2} is used to maintain a diverse set of
		nondominated candidate architectures rather than reducing the objectives to a
		single predetermined scalarization.
		
		An important feature of this problem is the asymmetry in evaluation cost. Resource
		quantities such as memory, latency, and energy can be obtained efficiently from
		the architecture encoding and a calibrated hardware cost model~\cite{burrello2023constraints},
		whereas estimating predictive performance requires candidate training and therefore
		dominates the evaluation cost. Consequently, the central challenge is not to
		approximate every objective equally, but to obtain an inexpensive yet sufficiently
		reliable estimate of the expensive performance objective. This observation
		motivates the teacher-guided low-fidelity evaluation developed in this work.
		
		\subsection{Capacity Mismatch in Low-Fidelity Ranking}
		
		NSGA-II uses objective values in two different ways. Nondominated sorting compares candidates and asks whether one dominates another. Crowding distance, however, also depends on the numerical spacing between objective values and may change under a nonlinear transformation even when the order is preserved. Our theory therefore focuses on the comparison-based part of selection. A candidate-independent additive bias cancels in pairwise comparisons, whereas architecture-dependent bias can reverse them. Distortion of crowding distance is treated separately as a limitation. Existing studies already recognize the importance of ranking. Zero-cost proxies are commonly evaluated using Spearman correlation \cite{abdelfattah2021zerocost}, reduced-training studies report Kendall-$\tau$ retention \cite{kyriakides2020reduced,ning2021evaluating}, and several recent surrogates predict order directly \cite{jiang2025spnas,xue2024tcmr,tian2023pairwise}. The remaining gap appears one step earlier. The low-fidelity evaluations used to train these surrogates are usually accepted as given, and their ranking quality is measured only after they have been constructed. We are not aware of a low-fidelity mechanism designed with an explicit a priori bound on pairwise inversion and then linked to perturbation of the Pareto front.
		
		A naive low-fidelity evaluation can fail in a predictable way. When every candidate is trained briefly on a small subset, the resulting error is rarely uniform. Small candidates may underfit and appear weaker than they are, while large candidates may adapt quickly and appear stronger than they are. The proxy error then becomes correlated with model capacity, which is itself a variable explored by the search. As a result, the ranking is not simply noisy; it can be systematically tilted.
		
		We address this problem with a teacher-guided evaluator that uses knowledge distillation only as a short and inexpensive measurement procedure. We call it \emph{KD-Lite}: ``KD'' refers to knowledge distillation \cite{hinton2015kd}, while ``Lite'' emphasizes that each candidate receives only a small, capped number of updates on compact data. The complete Teacher-Guided Learning (TGL) score combines KD-Lite with stratified evaluation. A Gaussian process (GP) then contributes information from architectures that have already been fully evaluated. The method follows four connected steps. First, a pretrained teacher with Monte Carlo dropout \cite{gal2016dropout} uses confidence, margin, and entropy \cite{weinshall2018curriculum} to organize samples by difficulty. Second, two disjoint compact sets are drawn from strata defined by both difficulty and class. Third, KD-Lite trains each candidate from easy to hard, using a balance between task loss and distillation that may vary with capacity (Proposition~\ref{prop:alpha}). Finally, the TGL and GP scores are fused using residual statistics rather than a manually chosen schedule (Proposition~\ref{prop:beta}). The fused score determines which offspring receive expensive full evaluation.
		
		\subsection{Contributions}
		
		These observations lead to the following contributions.
		
		\begin{enumerate}
			\item The low-fidelity evaluation problem is reformulated as a \emph{ranking-reliability} task, which is delineated from the already well-established problem of learning a rank-aware surrogate model.
			\item TGL-NSGA-II and its low-fidelity evaluator, KD-Lite, are introduced. The latter employs short curriculum-based knowledge distillation on a compact stratified training set, while an independent stratified set is reserved for evaluation. The resulting teacher-guided score is then fused with a Gaussian-process surrogate to determine which candidates are subjected to expensive full evaluation.
			\item A population-level proof chain is established for a \emph{fixed candidate set}. It is shown that stratified evaluation reduces sampling variance; from this, a score-concentration bound is derived. This concentration result is then leveraged to yield pairwise inversion and expected Kendall-$\tau$
			bounds, which in turn are used to control first-front identification. Additionally, the perturbation in hypervolume induced by substituting true accuracy with proxy accuracy on the same candidate set is bounded.
			\item Two implementation rules are derived. The first rule provides a clipped, variance-aware fusion coefficient for the surrogate and proxy following the centering of their respective errors. The second rule furnishes a first-order capacity-adaptive distillation mechanism, which removes the component of signed proxy bias linearly correlated with model capacity, provided that the required coefficient is feasible.
			Furthermore, the scope of the theoretical guarantees is explicitly stated. These guarantees analyse low-fidelity evaluation and population-level selection reliability.
		\end{enumerate}

		The remainder of the paper develops this argument in stages. Section~\ref{sec:related} positions the work in the literature, Section~\ref{sec:problem} defines the constrained search problem, and Section~\ref{sec:method} presents the algorithm. Section~\ref{sec:theory} develops the ranking analysis, after which the experiments and discussion connect the theoretical quantities to the observed search behaviour.
		
		\section{Related Work and Research Gap}
		\label{sec:related}
		
		The proposed method lies at the intersection of three research threads: surrogate-assisted combinatorial optimization, low-fidelity performance estimation, and teacher-guided learning for constrained TinyML. Reviewing them together makes the remaining gap clear.
		
		\textbf{Expensive combinatorial optimization and surrogate-assisted NAS.}
		The closest problem setting is expensive combinatorial optimization. Liu et al.\ \cite{liu2024ecopsurvey} survey SAEAs for ECOPs and show that discrete problems remain less developed than the mature continuous branch \cite{cui2022saeo}. Han and Wang \cite{han2021rfea} make an observation that is central to this work: when a surrogate approximates a constraint, its errors can misguide the search. This is fundamentally a ranking failure rather than only a value-accuracy failure. Gu et al.\ address expensive constrained discrete problems \cite{gu2021discrete} and large-scale binary optimization \cite{gu2025sans} through surrogate and resource-allocation design. Architecture-search research has similarly moved from value prediction toward order prediction, including dominance-relation classifiers \cite{pan2019classification}, pairwise comparators \cite{tian2023pairwise}, order-preserving score predictors \cite{jiang2025spnas}, and group-rank models with online refinement \cite{xue2024tcmr}. Other work improves sample efficiency using architecture embeddings \cite{fan2022saenasne} or meta-learned few-epoch screening \cite{li2025metaknowledge}. Model-based approaches such as ParEGO \cite{knowles2006parego}, MOEA/D-EGO \cite{zhang2010moeadego}, and GP-UCB acquisition \cite{srinivas2010gpucb} fit Gaussian processes to evaluations collected during the search. Lyu et al.\ \cite{lyu2024surrogate} address surrogate cold start using transfer stacking and knee-region distillation, while learned models have also been combined with NSGA-II for CNN tuning \cite{gayibov2025interactive}. These methods carefully model ranking at the surrogate stage, but they generally accept the low-fidelity evaluations used to train the surrogate as given.
		
		\textbf{Low-fidelity estimation.}
		Within this broader literature, low-fidelity estimators are usually assessed empirically. Zero-cost proxies report Spearman $\rho\!\approx\!0.82$ \cite{abdelfattah2021zerocost}, and reduced-epoch training reports Kendall $\tau_b\!>\!0.7$ \cite{kyriakides2020reduced}. However, their bias and variance are typically analysed only after the estimator has been chosen \cite{ning2021evaluating}. Training-speed estimators \cite{ru2021speedy} and neural-tangent-kernel estimators at initialization \cite{mok2022ntk} follow the same pattern. Zhao et al.\ \cite{zhao2023dele} further show that unsuitable low-fidelity information can actively reduce predictor quality in a manner that depends on the search space. The remaining gap is therefore narrow but important: low-fidelity labels are rarely constructed with an a priori inversion bound, and we did not find a result that propagates their comparison error to perturbation of a population-level Pareto front.
		
		\textbf{Curriculum, distillation, and constrained TinyML search.}
		Curriculum learning and knowledge distillation provide two further pieces of this study. Curriculum theory predicts faster convergence under a useful ordering \cite{weinshall2018curriculum}, although its generalization benefit is conditional rather than universal \cite{arora2025whencl,saglietti2022analytical}. The idea can be viewed as numerical continuation \cite{allgower2003continuation} and dates to Bengio et al.\ \cite{bengio2009curriculum}; progressive distillation can also induce an implicit curriculum \cite{panigrahi2024progressive}. Distillation theory studies supervision complexity \cite{harutyunyan2023supervision} and the limits of matching a teacher \cite{stanton2021doeskd}. Curriculum and distillation have been combined in several forms \cite{wang2022tc3kd,zhao2021instance,panagiotatos2019curriculum,das2026curriculum}, including multi-objective losses \cite{hayder2026dtokd,tian2021modiv} and representational alignment through CKA \cite{kornblith2019cka}. These methods aim to train a better final student. In contrast, we use short distillation as a measuring instrument that feeds an evolutionary selection operator. For deployment constraints, constrained multi-objective evolutionary algorithms offer adaptive penalties \cite{deb2000constraint}, stochastic ranking \cite{runarsson2000stochastic}, and $\varepsilon$-relaxation \cite{takahama2010constrained}. We use an adaptive penalty that combines naturally with a scalar promotion score. Similar multi-objective formulations have been effective for constrained embedded design on physiological signals \cite{shaikh2024myoelectric}. Hardware-aware microcontroller NAS is also well developed \cite{burrello2023constraints,deutel2023mobo,gambella2025nachos,bouzidi2024sonata,zhang2020hurricane,saha2026tinytnas}, including surrogate-initialization schemes for tight evaluation budgets \cite{garai2026oasi}. Here again, however, the fitness estimate is usually treated as an input rather than designed for ranking reliability.
		
		Table~\ref{tab:gap} summarizes the specific distinction used in this paper.
		
		\begin{table*}[!t]
			\centering
			\caption{Positioning of the proposed method relative to representative literature groups. HV denotes hypervolume.}
			\label{tab:gap}
			\footnotesize
			\renewcommand{\arraystretch}{1.18}
			\setlength{\tabcolsep}{5pt}
			\begin{tabularx}{\textwidth}{@{}
					>{\raggedright\arraybackslash}X
					>{\centering\arraybackslash}p{2.55cm}
					>{\centering\arraybackslash}p{2.55cm}
					>{\centering\arraybackslash}p{2.55cm}@{}}
				\toprule
				\textbf{Literature group} &
				\shortstack{\textbf{Rank-aware}\\\textbf{by design}} &
				\shortstack{\textbf{A priori}\\\textbf{error bound}} &
				\shortstack{\textbf{Linked to}\\\textbf{HV perturbation}} \\
				\midrule
				Surrogate-assisted evolutionary NAS (ENAS) \cite{jiang2025spnas,xue2024tcmr,tian2023pairwise,pan2019classification}
				& Surrogate only & No & No \\[2pt]
				Low-fidelity and zero-cost proxies \cite{abdelfattah2021zerocost,ning2021evaluating,kyriakides2020reduced,ru2021speedy}
				& Measured post hoc & No & No \\[2pt]
				Curriculum learning with KD \cite{wang2022tc3kd,zhao2021instance,panigrahi2024progressive}
				& Not applicable (training) & Not applicable & No \\[2pt]
				SAEAs for ECOPs \cite{liu2024ecopsurvey,han2021rfea,gu2021discrete,gu2025sans}
				& Surrogate only & No & No \\[2pt]
				Hardware-aware NAS \cite{burrello2023constraints,deutel2023mobo,gambella2025nachos}
				& Assumed input & No & No \\
				\midrule
				\textbf{This work}
				& \textbf{Proxy itself}
				& \textbf{Yes (Theorem~\ref{thm:rank})}
				& \textbf{Yes (Proposition~\ref{prop:hv})} \\
				\bottomrule
			\end{tabularx}
		\end{table*}
		
		Taken together, the literature shows that ranking, low-fidelity evaluation, distillation, and constrained search are individually well studied. What is missing is a direct connection between them. Our contribution is to design the low-fidelity evaluation around explicit analyses of sampling variance and pairwise inversion, and then connect the resulting comparison errors to the first nondominated front of a fixed candidate population. The next section formalizes the optimization problem for which this connection is needed.
		
		\section{Problem Formulation}
		\label{sec:problem}
		
		We now make the setting precise before developing the method. Let $\Aspace$ be a finite architecture space. Each $a\in\Aspace$ is a fixed-length vector of \emph{discrete} design variables: depth, per-layer width, kernel size, normalization flags, dense-layer count and quantization choice. With $L_{\mathrm{arch}}$ layers drawn from $|V|$ per-layer options plus the global choices, $|\Aspace|$ grows as $O(|V|^{L_{\mathrm{arch}}})$, so even the modest space of Table~\ref{tab:config} holds on the order of $10^5$ feasible-in-principle designs. The problem is therefore combinatorial in the usual sense, and expensive in the sense of \cite{liu2024ecopsurvey}: no closed form links the encoding to the accuracy objective, which is available only through a training run. Let $\Dpool$ be the labelled pool. Full training gives $\theta^\star(a)\in\arg\min_\theta\mathcal{L}(\theta;a,\Dpool)$, after which
		\begin{equation}
			\mathbf{f}(a)=\big[e(a),\,M(a),\,L(a),\,E_{\mathrm{inf}}(a)\big]^{\!\top}\in\R^{d_{\mathrm{obj}}},\quad d_{\mathrm{obj}}=4,
		\end{equation}
		with $e(a)\in[0,1]$ the validation error, $M(a)$ the peak memory, $L(a)$ the latency, and $E_{\mathrm{inf}}(a)$ the per-inference energy. Write $u(a)=1-e(a)$.
		
		We use multi-objective neural architecture search (MONAS) to denote architecture search with more than one objective.
		
		\begin{definition}[Constrained MONAS problem]
			\label{def:cmonas}
			$\min_{a\in\Aspace}\mathbf{f}(a)$ subject to $g_j(a)\le0$ for $j=1,\dots,J$, where the core constraints encode resource limits such as $M(a)\le M_{\max}$, $L(a)\le L_{\max}$, and an energy or flash-memory ceiling. In the audio implementation, false-positive rate (FPR) is additionally used as an application-level screening and feasibility metric. Its low-fidelity estimation error is evaluated empirically and is not part of the single-coordinate ranking theorem. The total constraint violation (CV) is
			\begin{equation}
				\CV(a)=\sum\nolimits_{j}\max\{0,\,g_j(a)\},
				\label{eq:cv}
			\end{equation}
			and $a$ is feasible when $\CV(a)=0$. The solution is the feasible Pareto set, scored by the hypervolume of its feasible members \cite{zitzler2003indicators}.
		\end{definition}
		
		For the theoretical analysis, $M$, $L$, and $E_{\mathrm{inf}}$ are treated as known from the encoding and the calibrated cost model, so only the accuracy coordinate is modelled as uncertain. The practical audio implementation additionally predicts FPR during promotion; true FPR is used after full evaluation when reporting feasibility.
		
		\begin{definition}[Proxy, induced order, ranking quality]
			\label{def:rank}
			A proxy is a map $S:\Aspace\to\R$ meant to increase with $u$. For a population $\mathcal{P}$ of size $P$, the pair $(a_i,a_j)$ is \emph{inverted} when $\big(S(a_i)-S(a_j)\big)\big(u(a_i)-u(a_j)\big)<0$. With $D$ the number of inverted pairs,
			\begin{equation}
				\tau(\mathcal{P})=1-2D\big/\tbinom{P}{2}.
				\label{eq:tau-def}
			\end{equation}
		\end{definition}
		
		Equation~\eqref{eq:tau-def} is the ranking measure commonly reported in the proxy literature. Having defined both the constrained search problem and its ranking criterion, we can now construct a low-fidelity evaluator whose design leads to a lower bound on $\E[\tau]$, rather than relying only on post hoc measurement.
		
		\section{Proposed TGL-NSGA-II Framework}
		\label{sec:method}
		
		With the ranking criterion now defined, we turn to the complete TGL-NSGA-II procedure. Its purpose is to reduce expensive full-training evaluations while preserving the comparisons needed by evolutionary selection. The method uses two complementary low-cost signals. KD-Lite supplies a teacher-guided score even when only a few architectures have been fully evaluated. A Gaussian-process (GP) surrogate then learns from those fully evaluated architectures and becomes more informative as the archive grows. Their fused score is used only to decide which offspring should be promoted to expensive full evaluation.
		
		The complete procedure has an offline stage and a generation-wise search stage, as shown in Fig.~\ref{fig:method}. In the offline stage, a fixed pretrained teacher is used only once to estimate the difficulty of the available samples. The data are then divided jointly by difficulty and class, and two disjoint compact sets are drawn: a proxy-training set $\Dtr$ and a proxy-evaluation set $\Dev$. During the evolutionary search, every offspring receives a short teacher-guided evaluation on these compact sets. Only a small fraction of promising or uncertain offspring are then fully trained. Their true objective values are added to the archive $\mathcal{E}$ and are used to update the GP surrogate. In this way, the teacher is most useful when the archive is small, while the surrogate becomes more informative as the search progresses.
		
		\begin{figure}[!htbp]
			\centering
			\begin{tikzpicture}[
				box/.style={draw, rounded corners=1.4pt, align=center, inner sep=2.2pt,
					font=\scriptsize, minimum height=5.0mm, text width=18mm},
				wide/.style={box, text width=41mm},
				half/.style={box, text width=19mm},
				ar/.style={-{Latex[length=1.5mm]}, thin},
				fb/.style={-{Latex[length=1.5mm]}, thin, dashed},
				lbl/.style={font=\scriptsize\itshape, inner sep=1pt}
				]
				\node[lbl] (lab_a) at (0,0) {(a) once per search};
				\node[box, below=1.5mm of lab_a.south west, anchor=north west] (pool) {Pool $\Dpool$};
				\node[box, right=5mm of pool] (teach) {Teacher $T_\phi$\\ MC dropout $\times K$};
				\node[box, below=3.4mm of teach] (score) {difficulty $d_i$\\ from $c_i,m_i,\widetilde H_i$};
				\node[box, left=5mm of score] (strat) {joint strata:\ difficulty $\times$ class};
				\node[wide, below=3.2mm of strat.south west, anchor=north west] (dsm)
				{compact sets: proxy train $\Dtr$; proxy evaluate $\Dev$};
				\draw[ar] (pool) -- (teach);
				\draw[ar] (teach) -- (score);
				\draw[ar] (score) -- (strat);
				\draw[ar] (strat) -- (strat|-dsm.north);
				\begin{scope}[on background layer]
					\node[draw, dashed, rounded corners=2pt, fit=(lab_a)(pool)(teach)(score)(strat)(dsm),
					inner sep=2.4pt] (offline) {};
				\end{scope}
				
				\node[lbl, below=3.0mm of offline.south west, anchor=north west] (lab_b) {(b) each generation};
				\node[box, below=1.5mm of lab_b.south west, anchor=north west] (pop) {population $\mathcal{P}_t$};
				\node[box, right=5mm of pop] (var) {SBX $+$ poly.\ mutation};
				\node[box, right=5mm of var] (off) {offspring $\mathcal{Q}_t$};
				\draw[ar] (pop) -- (var);
				\draw[ar] (var) -- (off);
				
				\node[half, minimum height=13.5mm, below=4.6mm of pop.south west, anchor=north west] (gp)
				{GP surrogate\\ $\widehat{\mathrm{ACC}},\sigma_{\mathrm{gp}},\widehat{\CV}$\\ $\Rightarrow S_{\mathrm{gp}}$};
				\node[half, minimum height=13.5mm, below=4.6mm of off.south east, anchor=north east] (tgl)
				{KD-Lite\\ short curriculum on $\Dtr$\\ evaluate on $\Dev$\\ $\Rightarrow S_{\mathrm{tgl}}$};
				\draw[ar] (off.south) -- ++(0,-2.3mm) -| (gp.north);
				\draw[ar] (off.south) -- ++(0,-2.3mm) -| (tgl.north);
				
				\node[wide, below=4.0mm of gp.south west, anchor=north west] (fuse)
				{normalize $+$ fuse\ $\beta^\star$; add $\zeta_t\widehat{\CV}$};
				\node[wide, below=2.8mm of fuse] (final)
				{promotion score: predicted quality $+$ GP uncertainty};
				\node[half, below=3.4mm of final.south west, anchor=north west] (top)
				{top-$k\%$ $\to$ full training};
				\node[half, right=2.4mm of top] (upd) {update archive $\mathcal{E}$};
				\node[half, right=2.4mm of upd] (env) {constrained sort $+$ crowding};
				\draw[ar] (gp.south) -- (gp.south|-fuse.north);
				\draw[ar] (tgl.south) -- (tgl.south|-fuse.east) -- (fuse.east);
				\draw[ar] (fuse) -- (final);
				\draw[ar] (final.south) -- ++(0,-1.5mm) -| (top.north);
				\draw[ar] (top) -- (upd);
				\draw[ar] (upd) -- (env);
				\begin{scope}[on background layer]
					\node[draw, dashed, rounded corners=2pt,
					fit=(lab_b)(pop)(var)(off)(gp)(tgl)(fuse)(final)(top)(upd)(env),
					inner sep=3.0pt] (online) {};
				\end{scope}
				\draw[fb] (upd.south) -- ++(0,-1.8mm)
				-| ([xshift=-3.0mm]online.west|-upd)
				-- ([xshift=-3.0mm]online.west|-gp) -- (gp.west);
				\draw[ar] (env.south) -- ++(0,-3.6mm)
				-| ([xshift=-7.0mm]online.west|-env)
				-- node[lbl, left=0pt, pos=0.62] {$\mathcal{P}_{t+1}$}
				([xshift=-7.0mm]online.west|-pop) -- (pop.west);
				\draw[fb] (dsm.east) -- ([xshift=6.5mm]dsm.east) |- (tgl.east);
			\end{tikzpicture}
			\caption{Overview of TGL-NSGA-II. The teacher is used once to score sample difficulty and to form the disjoint proxy-training and proxy-evaluation sets. In each generation, every offspring receives a low-cost KD-Lite score and a GP-surrogate score. Their fused score, together with GP uncertainty, determines which offspring are promoted to expensive full training. Only fully evaluated promoted offspring update the archive and enter environmental selection with the parent population.}
			\label{fig:method}
		\end{figure}
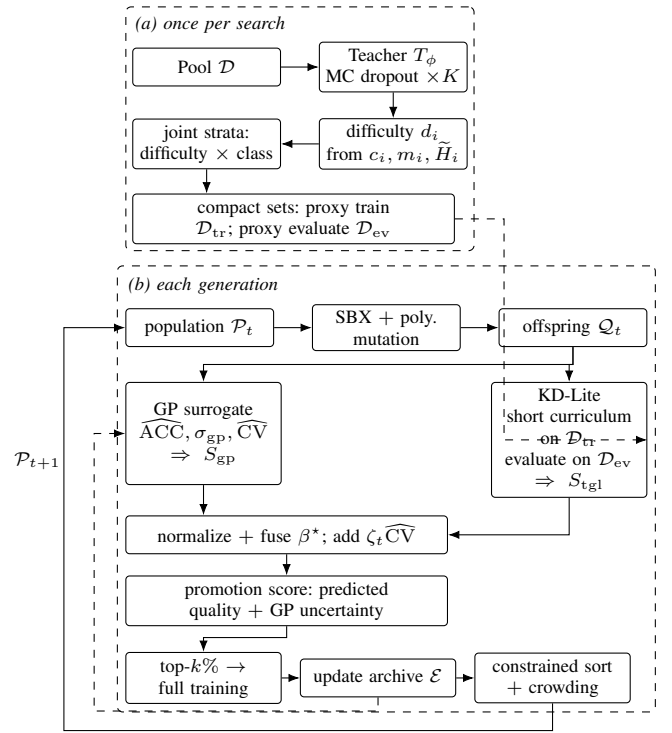
		
		The teacher $T_\phi$ first provides a stable easy-to-hard ordering of samples. Since a single deterministic prediction can be overconfident, we use Monte Carlo dropout \cite{gal2016dropout}: for sample $i$, the teacher is run $K$ times with dropout and averaged to $\bar p_i$. From $\bar p_i$ we compute confidence $c_i=\max_c\bar p_i(c)$, top-two margin $m_i=\bar p_i(c_{(1)})-\bar p_i(c_{(2)})$, and class-normalized entropy $\widetilde H_i=H_i/\log C_{\mathrm{cls}}$, all in $[0,1]$, and combine them into a single difficulty score,
		\begin{equation}
			d_i=\omega_1(1-c_i)+\omega_2(1-m_i)+\omega_3\widetilde H_i,
			\qquad \textstyle\sum_j\omega_j=1,
			\label{eq:difficulty}
		\end{equation}
		where $\boldsymbol{\omega}=[\omega_1,\omega_2,\omega_3]\in[0,1]^3$ weight the three uncertainty cues; their default calibration is discussed in Section~\ref{sec:setup}. Larger $d_i$ means lower teacher certainty. This score is used only to order and stratify samples. It is neither ground-truth difficulty nor architecture fitness. Because it does not use the true class, it may also miss a confidently wrong teacher. Table~\ref{tab:mechanism} examines this failure mode through uniformly and selectively weak teacher controls.
		
		\subsection{Joint Stratification and Compact Sets}
		
		The difficulty score becomes useful once it is combined with class information. This produces compact data sets that lower the proxy-training cost while retaining the stratified structure required by the theory.
		
		\begin{definition}[Joint strata and compact sets]
			\label{def:strata}
			Sort $\Dpool$ by $d_i$ into three difficulty tiers ($S_{\mathrm{easy}},S_{\mathrm{med}},S_{\mathrm{hard}}$), and then split each tier by class to obtain cell $\Dpool_{rc}$ with weight $w_{rc}=|\Dpool_{rc}|/|\Dpool|$. Draw two disjoint compact sets from these cells: the proxy-training set $\Dtr$ of size $n_{\mathrm{tr}}$ and the proxy-evaluation set $\Dev$ of size $n_{\mathrm{ev}}$. Proportional allocation gives $|\Dtr^{(rc)}|\approx w_{rc}n_{\mathrm{tr}}$ and $|\Dev^{(rc)}|\approx w_{rc}n_{\mathrm{ev}}$, with uniform sampling without replacement within each cell and $\Dtr\cap\Dev=\varnothing$. By default, $n_{\mathrm{tr}}=\varphi|\Dpool|$ with $\varphi=0.03$.
		\end{definition}
		
		Proportional allocation preserves the joint composition of class and difficulty without forcing all cells to have the same size. Variance is reduced by estimating performance within each stratum and then recombining the estimates using $w_{rc}$. The same difficulty tiers naturally provide the curriculum used by KD-Lite.
		
		\subsection{KD-Lite Fitness Approximation}
		\label{sec:kdlite}
		
		The teacher now takes a second role by providing soft supervision. In \emph{KD-Lite}, each candidate is trained only on the compact set $\Dtr$, and the teacher probabilities are cached once. By default, $|\Dtr|$ is only $3\%$ of the full data pool. KD-Lite is therefore a low-fidelity evaluator rather than a complete distillation procedure. It follows an easy-to-hard curriculum $\Dtr^{(1)}\subset\Dtr^{(2)}\subset\Dtr^{(3)}=\Dtr$, corresponding to easy samples, easy and medium samples, and finally all three tiers. At stage $s$, $\theta$ is updated by minimizing
		\begin{multline}
			F_s(\theta)=\frac{1}{|\Dtr^{(s)}|}\sum_{x\in\Dtr^{(s)}}\Big[
			\big(1-\alpha(a)\big)\,\ell_{\mathrm{CE}}(\theta;x)\\
			+\alpha(a)T_{\mathrm{KD}}^2D_{\mathrm{KL}}\!\left(p_T^{T_{\mathrm{KD}}}(x)\,\|\,p_\theta^{T_{\mathrm{KD}}}(x)\right)\Big],
			\label{eq:kdlite}
		\end{multline}
		where $\ell_{\mathrm{CE}}$ is cross-entropy, $D_{\mathrm{KL}}$ is KL divergence, and $p_T^{T_{\mathrm{KD}}}$ and $p_\theta^{T_{\mathrm{KD}}}$ are the teacher and candidate distributions at temperature $T_{\mathrm{KD}}$. The factor $T_{\mathrm{KD}}^2$ follows standard distillation practice \cite{hinton2015kd}, while $\alpha(a)\in[0,1]$ balances hard-label supervision against teacher supervision. Each stage stops when its validation loss no longer improves, using a patience of one epoch, but it is also capped at $t_{\max}=15$ SGD steps. If stage $s$ uses $t_s\le t_{\max}$ steps, then the total proxy budget is $T_c=t_1+t_2+t_3\le 3t_{\max}=45$. Easy stages often converge well before the cap, so the realized value of $T_c$ is typically smaller. Each stage warm-starts the next, and the final short-trained parameters are denoted by $\theta_c(a)$. The cap ensures $T_c\ll T_{\mathrm{full}}$ even when an individual candidate converges slowly.
		
		The distillation weight may depend on capacity: with $\pi(a)\in[0,1]$ a normalized capacity measure,
		\begin{equation}
			\alpha(a)=\alpha_0+\alpha_1\pi(a),
			\label{eq:alpha-method}
		\end{equation}
		clipped to $[0,1]$. Rather than assuming larger models need more supervision, $\alpha_1$ is calibrated on a small pilot set to reduce capacity-dependent proxy bias; Proposition~\ref{prop:alpha} justifies this, and the useful sign of $\alpha_1$ follows the measured bias trend.
		
		After KD-Lite training, $\theta_c(a)$ is scored on the independent set $\Dev$. Keeping $\Dev$ separate from $\Dtr$ avoids evaluating a candidate on the samples used for its short training and makes the sampling component of the score directly analyzable. The next subsection turns these evaluation outputs into the practical teacher-guided score.
		
		\subsection{Teacher-Guided Fitness Score}
		
		The practical TGL score combines model fidelity with application-level behaviour. It is kept separate from the simpler accuracy statistic used in the theory. Let $\ell_{\mathrm{CE}}(a)$ and $\ell_{\mathrm{KL}}(a)$ denote the mean cross-entropy and the KL divergence between teacher and candidate on $\Dev$. Let $\mathrm{CKA}(a)$ denote the centered kernel alignment \cite{kornblith2019cka} between their cached activations. The lower-is-better fidelity term is
		\begin{equation}
			\begin{aligned}
				S_{\mathrm{fid}}(a)
				&=\psi_1\widetilde\ell_{\mathrm{CE}}(a)
				+\psi_2\widetilde\ell_{\mathrm{KL}}(a)\\
				&\quad+\psi_3\bigl(1-\mathrm{CKA}(a)\bigr),
				\qquad \textstyle\sum_j\psi_j=1,
			\end{aligned}
			\label{eq:sfid}
		\end{equation}
		where $\boldsymbol{\psi}\in[0,1]^3$ and tildes denote minimum-maximum normalization over the offspring set. CKA is a practical cue and is not assumed by the theory. The application term uses relative FPR,
		\begin{equation}
			\begin{aligned}
				S_{\mathrm{perf}}(a)
				&=\frac{\mathrm{FPR}_{S}(a)-\mathrm{FPR}_{T}}
				{\max\{\mathrm{FPR}_{T},\epsilon_0\}}.
			\end{aligned}
			\label{eq:sperf}
		\end{equation}
		and the two combine as $S_{\mathrm{tgl}}(a)=\rho\,\widetilde S_{\mathrm{fid}}(a)+(1-\rho)\widetilde S_{\mathrm{perf}}(a)$ (lower better), where $\rho\in[0,1]$ balances fidelity against application behaviour and $S_{\mathrm{perf}}$ is replaced or omitted for tasks without a meaningful FPR. For the theoretical analysis we use instead the stratified accuracy statistic
		\begin{equation}
			S(a)=\sum_{r,c}w_{rc}\,\bar g_{rc}(a),
			\label{eq:score-analysis}
		\end{equation}
		where $\bar g_{rc}(a)$ is the mean correctness of $\theta_c(a)$ on $\Dev^{(rc)}$. Thus, $S_{\mathrm{tgl}}$ is the practical multi-signal promotion score, while $S(a)$ is the estimator analysed in Section~\ref{sec:theory}. The teacher-guided score is immediately available for new candidates; the GP surrogate described next contributes information accumulated from earlier full evaluations.
		
		\subsection{Surrogate Fusion and Candidate Promotion}
		
		KD-Lite provides candidate-specific information even during surrogate cold start. The GP surrogate complements this signal by learning from the archive $\mathcal{E}$ of fully evaluated architectures. For the audio tasks, it predicts accuracy $\widehat{\mathrm{ACC}}(a)$, FPR $\widehat{\mathrm{FPR}}(a)$, and posterior uncertainty $\sigma_{\mathrm{gp}}(a)$. Cheap hardware quantities are obtained separately from the calibrated cost model. The lower-is-better surrogate quality score is
		\begin{equation}
			S_{\mathrm{gp}}(a)=
			\nu\big(1-\widehat{\mathrm{ACC}}(a)\big)
			+(1-\nu)\widehat{\mathrm{FPR}}(a),
			\quad \nu\in[0,1],
			\label{eq:sgp}
		\end{equation}
		For tasks without an FPR objective, the FPR term is omitted by setting $\nu=1$. Predicted constraint violations are summarized by $\widehat{\CV}(a)$. After minimum-maximum normalization of $S_{\mathrm{gp}}$ and $S_{\mathrm{tgl}}$ over the current offspring set, the fused lower-is-better screening score is
		\begin{equation}
			S_{\mathrm{c}}(a)=
			\beta^\star\widetilde S_{\mathrm{gp}}(a)
			+(1-\beta^\star)\widetilde S_{\mathrm{tgl}}(a)
			+\zeta_t\,\widetilde{\widehat{\CV}}(a),
			\label{eq:fusion}
		\end{equation}
		where $\beta^\star\in[0,1]$ is the clipped residual-variance weight of Proposition~\ref{prop:beta}, and $\zeta_t\ge0$ is a generation-dependent constraint penalty (small early to allow exploration near the boundary, larger later). Predicted quality and GP uncertainty then form the higher-is-better promotion score
		\begin{equation}
			S_{\mathrm{f}}(a)=
			(1-\eta_e)\big(1-\widetilde S_{\mathrm{c}}(a)\big)
			+\eta_e\,\widetilde\sigma_{\mathrm{gp}}(a),
			\quad \eta_e\in[0,1],
			\label{eq:sfinal}
		\end{equation}
		This follows the GP-UCB intuition \cite{srinivas2010gpucb}. The exploration weight $\eta_e$ balances predicted quality against the GP uncertainty bonus. The top $k\%$ of offspring, with a default of $20\%$, are promoted to full training. Only these promoted offspring receive true objective values and enter constrained nondominated sorting with the parent population. The remaining offspring are discarded rather than entering selection with proxy values. Consequently, the archive, the population, and the reported Pareto front remain on the same fully evaluated scale. Algorithm~\ref{alg:main} summarizes the complete procedure, including offspring generation by SBX and polynomial mutation.
		
		\begin{algorithm}[!t]
			\caption{TGL-NSGA-II}
			\label{alg:main}
			\begingroup
			\scriptsize
			\renewcommand{\baselinestretch}{0.88}\selectfont
			\algrenewcommand\algorithmicindent{0.85em}
			\setlength{\topsep}{1pt}
			\setlength{\partopsep}{0pt}
			\setlength{\itemsep}{0pt}
			\setlength{\parsep}{0pt}
			\begin{algorithmic}[1]
				\Require population size $P$, number of generations $G$, teacher $T_\phi$, hardware cost model $C$, promotion fraction $k$
				\Ensure feasible Pareto set
				\State \textbf{Offline:} run $T_\phi$ on $\Dpool$; compute $d_i$ using \eqref{eq:difficulty}; build joint strata; draw disjoint $\Dtr$ and $\Dev$; cache teacher outputs
				\State initialize population $\mathcal{P}_0$; fully evaluate it; initialize archive $\mathcal{E}$
				\For{$t=0$ to $G-1$}
				\State fit the GP surrogate on $\mathcal{E}$; estimate $\beta^\star$; update $\zeta_t$
				\State generate offspring $\mathcal{Q}_t$ from $\mathcal{P}_t$ by SBX and polynomial mutation
				\ForAll{$a\in\mathcal{Q}_t$}
				\State obtain cheap hardware objectives and constraints from $C(a)$
				\State run KD-Lite on $\Dtr$ to obtain $\theta_c(a)$
				\State evaluate $\theta_c(a)$ on $\Dev$ and compute $S_{\mathrm{tgl}}(a)$
				\State compute $S_{\mathrm{gp}}(a)$, $S_{\mathrm{c}}(a)$, and $S_{\mathrm{f}}(a)$
				\EndFor
				\State let $\mathcal{Q}^{\mathrm{full}}_t$ be the top $k\%$ offspring by $S_{\mathrm{f}}$; fully evaluate them; add them to $\mathcal{E}$
				\State form $\mathcal{P}_{t+1}$ by constrained nondominated sorting and crowding on $\mathcal{P}_t\cup\mathcal{Q}^{\mathrm{full}}_t$
				\EndFor
			\end{algorithmic}
			\endgroup
			\vspace{-0.8ex}
		\end{algorithm}
		
		\subsection{Computational Cost}
		
		Having described how candidates are promoted, we can now account for the additional computation. A central design choice keeps the teacher away from the repeated search loop as much as possible. The MC-dropout difficulty pass runs once on $\Dpool$, after which the teacher probability vectors and required activations for the compact sets are cached. The one-time teacher cost is therefore $O(|\Dpool|)$ forward evaluations, scaled by the MC-dropout count $K$.
		
		In each generation, every offspring receives a KD-Lite evaluation. If $n_{\mathrm{tr}}=|\Dtr|$ and $n_{\mathrm{ev}}=|\Dev|$, the short-training cost is approximately $O(Pn_{\mathrm{tr}}T_c)$, followed by an evaluation cost proportional to $Pn_{\mathrm{ev}}$. Only the promoted fraction $kP$ receives full training, which gives an expensive-training cost of $O(kP T_{\mathrm{full}})$. The intended regime is $n_{\mathrm{tr}}T_c\ll T_{\mathrm{full}}$. For a standard implementation, GP fitting costs approximately $O(n_{\mathcal{E}}^3)$ with $n_{\mathcal{E}}=|\mathcal{E}|$, while NSGA-II environmental selection adds $O(P^2)$ per generation. TGL-NSGA-II therefore spends a small amount of computation on every offspring to avoid a full training run for poor candidates. The experiments report wall-clock cost so that this trade-off can be assessed directly. The next section examines what this low-fidelity procedure guarantees about candidate ordering.
		
		\section{Theoretical Analysis}
		\label{sec:theory}
		
		Having specified the algorithm, we now ask what can be guaranteed about the low-fidelity accuracy score that drives its promotion decisions. The scope is deliberately narrow: the results hold for a fixed architecture $a$ or a fixed finite population $\mathcal{P}$, and quantify sampling error, pairwise-comparison error, and the resulting perturbation of the first nondominated front. They do not establish convergence of the full multi-generation NSGA-II trajectory, since an incorrect decision in one generation changes which candidates are generated later. The analysis proceeds in four steps: (i) variance reduction via stratification, (ii) optimization error of KD-Lite, (iii) concentration and ranking bounds, and (iv) population-level front and hypervolume perturbation. Each step is self-contained, and the implementation rules are derived directly from the bounds.
		
		\subsection{Assumptions and Error Decomposition}
		
		For a fixed architecture $a$, let $\theta_c(a)$ denote the parameters obtained after the short curriculum training on $\Dtr$. Let
		\begin{equation}
			u_c(a)=\Pr_{(x,y)\sim\Dpool}\!\left[\hat y_{\theta_c(a)}(x)=y\right]
			\label{eq:uc}
		\end{equation}
		be the population accuracy of the proxy-trained candidate, and let $u(a)$ denote its full-training population accuracy defined in Section~\ref{sec:problem}. The stratified score $S(a)$ in \eqref{eq:score-analysis} estimates $u_c(a)$ using the independent set $\Dev$.
		
		We define the \emph{signed proxy bias}
		\begin{equation}
			B(a)=u_c(a)-u(a),
			\label{eq:signed-bias}
		\end{equation}
		and the centred evaluation noise
		\begin{equation}
			Z(a)=S(a)-u_c(a).
			\label{eq:eval-noise}
		\end{equation}
		Hence
		\begin{equation}
			S(a)=u(a)+B(a)+Z(a).
			\label{eq:score-decomp}
		\end{equation}
		This decomposition separates a systematic training bias $B(a)$ from the sampling fluctuation $Z(a)$. Increasing $n_{\mathrm{ev}}$ reduces $Z(a)$ but does not, by itself, remove $B(a)$.
		
		\begin{assumption}[Linearized short-training regime]
			\label{as:linear}
			During the $T_c$ proxy-training steps, the candidate logits are well approximated by an affine function of the trainable parameters on the domain visited by projected stochastic gradient descent (SGD). Under this approximation, the cross-entropy and forward KL distillation terms in \eqref{eq:kdlite} are convex in the optimized parameters.
		\end{assumption}
		
		\begin{assumption}[Teacher stratification is informative]
			\label{as:teacher}
			For the correctness statistic of a fixed proxy-trained candidate, at least two strata defined jointly by difficulty and class have different conditional means. Equivalently, the between-stratum variance defined in \eqref{eq:rho} is nonzero. If this assumption fails, stratification reduces to simple random sampling in expectation, so no degradation occurs. The assumption is testable from pilot evaluations and does not require the teacher to be perfectly accurate.
		\end{assumption}
		
		\begin{assumption}[Bounded stage drift and smoothness]
			\label{as:drift}
			For curriculum stage $s$, the linearized objective $F_s$ is $\mu$-strongly convex and $L$-smooth, with gradients bounded by $G_F$ on the optimization domain. If $\theta_s^\star$ is its minimizer, then
			$\|\theta_s^\star-\theta_{s+1}^\star\|\le\delta_s$.
		\end{assumption}
		
		\begin{assumption}[Margin transfer and bounded systematic bias]
			\label{as:bias}
			There exists a finite constant $c_m>0$ such that the optimization residual of the short-training problem changes population classification accuracy by at most $c_m\sqrt{\epsilon_{\mathrm{opt}}(a)}$ in the regime considered. Such margin bounds are standard in classification theory when the loss is margin-calibrated. In addition, the systematic difference between the proxy-training objective and full task training is bounded by a nonnegative teacher-mismatch term $\varepsilon_{\mathrm{KD}}(a)$ and a residual capacity term $\Gamma(a)$. Thus
			\begin{equation}
				|B(a)|\le b_{\max}(a)
				:=c_m\sqrt{\epsilon_{\mathrm{opt}}(a)}
				+\alpha(a)\varepsilon_{\mathrm{KD}}(a)
				+\Gamma(a).
				\label{eq:bias}
			\end{equation}
			Here $\varepsilon_{\mathrm{KD}}(a)$ measures mismatch between the teacher soft target and the task target in the distillation objective; it is not identified with the teacher's zero-one classification error.
		\end{assumption}
		
		\subsection{Variance Reduction by Stratification}
		
		The first step controls the sampling term $Z(a)$. We show that scoring a fixed proxy-trained candidate on proportionally stratified data is never worse, and generically better, than scoring it on a simple random draw of the same size.
		
		\begin{lemma}[Variance reduction under proportional stratification]
			\label{lem:strat}
			Fix the trained parameters $\theta_c(a)$ and define
			$g_a(x,y)=\mathbf{1}[\hat y_{\theta_c(a)}(x)=y]\in[0,1]$.
			Let $\mu_g$ and $\sigma^2$ be the pool mean and variance of $g_a$. For a partition into cells $\{\Dpool_\iota\}$ with weights $w_\iota$, cell means $\mu_\iota$, and within-cell variances $\sigma_\iota^2$, let $\bar g_{\mathrm{srs}}$ be the mean of $n_{\mathrm{ev}}$ simple-random-sampling draws and let $\bar g_{\mathrm{str}}$ be the proportionally stratified estimator. Ignoring the favourable finite-population correction for sampling without replacement,
			\begin{equation}
				\Var[\bar g_{\mathrm{srs}}]-\Var[\bar g_{\mathrm{str}}]
				=\frac{1}{n_{\mathrm{ev}}}\sum_\iota w_\iota(\mu_\iota-\mu_g)^2\ge0.
				\label{eq:var-gain}
			\end{equation}
			Define the stratification efficiency
			\begin{equation}
				\rstr=
				\frac{\sum_\iota w_\iota(\mu_\iota-\mu_g)^2}{\sigma^2}\in[0,1].
				\label{eq:rho}
			\end{equation}
			Then
			\begin{equation}
				\Var[\bar g_{\mathrm{str}}]
				=(1-\rstr)\frac{\sigma^2}{n_{\mathrm{ev}}}.
				\label{eq:str-var}
			\end{equation}
			If partition $\mathcal{B}$ refines partition $\mathcal{A}$, then
			$\rstr(\mathcal{B})\ge\rstr(\mathcal{A})$, provided proportional allocation remains feasible in the refined cells.
		\end{lemma}
		
		\begin{proof}
			Condition on $\theta_c(a)$. The law of total variance gives $\sigma^2=\sum_\iota w_\iota\sigma_\iota^2+\sum_\iota w_\iota(\mu_\iota-\mu_g)^2$. Under proportional allocation $n_\iota=w_\iota n_{\mathrm{ev}}$, so $\Var[\bar g_{\mathrm{str}}]=\sum_\iota w_\iota^2\sigma_\iota^2/n_\iota=\tfrac{1}{n_{\mathrm{ev}}}\sum_\iota w_\iota\sigma_\iota^2$. Subtracting this from $\sigma^2/n_{\mathrm{ev}}$ gives Eqs.~\eqref{eq:var-gain} and~\eqref{eq:str-var}. The refinement claim follows because conditional expectation onto a finer partition cannot reduce the between-cell variance $\Var(\E[g_a\mid Z])$, by the tower property and $L^2$ contraction.
		\end{proof}
		
		Lemma~\ref{lem:strat} concerns the independent evaluation set $\Dev$. It does \emph{not} claim that drawing a stratified proxy-training set automatically reduces the variance of a data-dependent trained predictor. This separation is why $\Dtr$ and $\Dev$ are defined independently.
		
		\subsection{KD-Lite Optimization Error}
		
		Having bounded the sampling term, we turn to the systematic term $B(a)$, which originates in the short curriculum training itself. The next lemma makes the KD-Lite objective analyzable by giving a per-stage optimization-error recursion under the linearized regime.
		
		\begin{lemma}[Convexity and staged recursion]
			\label{lem:stages}
			Under Assumption~\ref{as:linear}, $F_s$ in \eqref{eq:kdlite} is convex for every $\alpha\in[0,1]$ and distillation temperature $T_{\mathrm{KD}}>0$. Under Assumption~\ref{as:drift}, projected SGD with step size $\eta \le 1/L$ for $t_s$ steps at stage $s$, warm-started from the averaged iterate of stage $s-1$, satisfies
			\begin{equation}
				\epsilon_s\le
				\frac{1}{2\eta t_s}
				\left(
				\sqrt{\frac{2\epsilon_{s-1}}{\mu}}+\delta_{s-1}
				\right)^2
				+\frac{\eta G_F^2}{2},
				\label{eq:recursion}
			\end{equation}
			where
			$\epsilon_s=\E[F_s(\bar\theta_s)]-F_s(\theta_s^\star)$ and $\delta_0=0$.
			Let $\epsilon_{\mathrm{opt}}(a):=\epsilon_3$.
		\end{lemma}
		
		\begin{proof}
			Cross-entropy with a fixed target and the forward KL divergence $D_{\mathrm{KL}}(p_T^{T_{\mathrm{KD}}}\|p_\theta^{T_{\mathrm{KD}}})$ are both convex in the student logits (the teacher distribution is fixed); under Assumption~\ref{as:linear} the logits are affine in the parameters, so convexity is preserved. The recursion follows from the standard projected-SGD bound under $L$-smoothness and $\mu$-strong convexity: $\epsilon_s\le\|\theta_{s-1}-\theta_s^\star\|^2/(2\eta t_s)+\eta G_F^2/2$, the triangle inequality, and strong convexity $\|\theta_{s-1}-\theta_{s-1}^\star\|\le\sqrt{2\epsilon_{s-1}/\mu}$.
		\end{proof}
		
		The lemma provides an analyzable optimization residual. It does not state that curriculum ordering is always superior to random ordering. A curriculum is useful when successive stage minimizers are sufficiently close; otherwise the staged schedule may provide no optimization advantage.
		
		\subsection{Score Concentration}
		
		The two preceding results combine directly: the stratified variance of Lemma~\ref{lem:strat} controls the sampling fluctuation, while the optimization residual of Lemma~\ref{lem:stages} feeds the bias bound of Assumption~\ref{as:bias}. Together they bound how far the observed score can stray from true accuracy.
		
		\begin{theorem}[Conditional concentration of the TGL accuracy score]
			\label{thm:conc}
			Fix architecture $a$ and condition on the proxy-trained parameters $\theta_c(a)$ obtained from $\Dtr$. Let $S(a)$ be evaluated on the independent proportionally stratified set $\Dev$ of size $n_{\mathrm{ev}}$. Then, for any $\delta\in(0,1)$, with probability at least $1-\delta$ over the draw of $\Dev$,
			\begin{equation}
				|S(a)-u(a)|
				\le |B(a)|+\Psi(n_{\mathrm{ev}},\delta),
				\label{eq:conc}
			\end{equation}
			where
			\begin{equation}
				\Psi(n_{\mathrm{ev}},\delta)=
				\sqrt{
					\frac{2(1-\rstr)\sigma^2\log(2/\delta)}
					{n_{\mathrm{ev}}}
				}
				+\frac{2\log(2/\delta)}{3n_{\mathrm{ev}}}.
				\label{eq:psi}
			\end{equation}
			Under Assumption~\ref{as:bias}, $|B(a)|$ may be replaced by $b_{\max}(a)$ in \eqref{eq:bias}. Moreover, the centred sampling error $Z(a)$ is sub-Gaussian with variance proxy
			\begin{equation}
				\sEval^2=(1-\rstr)\sigma^2/n_{\mathrm{ev}}.
				\label{eq:sigma}
			\end{equation}
		\end{theorem}
		
		\begin{proof}
			From \eqref{eq:score-decomp}, $S(a)-u(a)=B(a)+Z(a)$. Conditional on $\theta_c(a)$, $g_a$ is fixed and $S(a)$ is a proportionally stratified mean whose variance is given by Lemma~\ref{lem:strat}. Bernstein's inequality for bounded independent stratum samples gives \eqref{eq:psi}, and the triangle inequality yields \eqref{eq:conc}. Hoeffding's lemma within each independent stratum, combined over weighted stratum means, gives the sub-Gaussian form in \eqref{eq:sigma}.
		\end{proof}
		
		Only the sampling term decreases with $n_{\mathrm{ev}}$; the systematic term $B(a)$ is produced by short training, distillation, and capacity, and its architecture-independent part cancels in pairwise comparisons.
		
		\subsection{Ranking Reliability}
		
		Concentration of a single score is only a means to the end that matters for selection: the reliability of \emph{comparisons}. We now split the signed bias into a common part and a candidate-dependent part, and show that only the latter can invert a pairwise comparison.
		
		\begin{theorem}[Rank inversion and expected Kendall-$\tau$]
			\label{thm:rank}
			For architecture $a$, write the signed bias as
			\begin{equation}
				B(a)=\bar B+\xi(a),
				\label{eq:bias-split}
			\end{equation}
			where $\bar B$ is common to all candidates in the population and
			$|\xi(a)|\le\bar\xi$, with $\bar\xi$ a known upper bound estimated from pilot data. For the population-level bound, $\sigma^2$ denotes a uniform upper bound on the conditional correctness variance across all candidates in $\mathcal{P}$, and $\rstr$ denotes a conservative lower bound on the stratification efficiency. Consider two candidates $a$ and $a'$ with true accuracy gap
			$\Delta=|u(a)-u(a')|$. Define
			$t=(\Delta-2\bar\xi)_+$. Then
			\begin{equation}
				\begin{aligned}
					\Pr[(a,a')\text{ is inverted}]
					&\le q(\Delta)\\
					&:=\min\!\left\{1,
					2\exp\!\left[-\frac{n_{\mathrm{ev}}t^2}
					{8(1-\rstr)\sigma^2}\right]\right\}.
				\end{aligned}
				\label{eq:rank}
			\end{equation}
			For a fixed population $\mathcal{P}$ of size $P$,
			\begin{equation}
				\E[\tau(\mathcal{P})]
				\ge
				1-\frac{2}{\binom{P}{2}}
				\sum_{i<j}q(\Delta_{ij}).
				\label{eq:tau}
			\end{equation}
			The common signed bias $\bar B$ does not appear in either bound.
		\end{theorem}
		
		\begin{proof}
			Assume without loss of generality that $u(a)>u(a')$. By \eqref{eq:score-decomp} and \eqref{eq:bias-split},
			\[
			S(a)-S(a')
			=
			\Delta+\xi(a)-\xi(a')+Z(a)-Z(a').
			\]
			The common term $\bar B$ cancels exactly. If the pair is inverted, then
			$Z(a')-Z(a)\ge t$. This event implies that either
			$Z(a')\ge t/2$ or $-Z(a)\ge t/2$. Therefore, by the union bound,
			\[
			\Pr[\text{inversion}]
			\le
			\Pr[Z(a')\ge t/2]+\Pr[-Z(a)\ge t/2].
			\]
			Each term is bounded using the sub-Gaussian tail from Theorem~\ref{thm:conc} with variance proxy $\sEval^2$ from \eqref{eq:sigma}. Clipping the resulting expression at one gives \eqref{eq:rank}. Equation~\eqref{eq:tau} follows from
			$\tau=1-2D/\binom{P}{2}$ and
			$\E[D]=\sum_{i<j}\Pr[(a_i,a_j)\text{ is inverted}]$.
		\end{proof}
		
		\begin{remark}[Uniform additive bias and differential bias]
			\label{rem:teacher}
			A candidate-independent additive bias does not affect dominance comparisons based on the accuracy coordinate because it cancels in every pairwise difference. Differential bias $\xi(a)$ can reverse such comparisons and therefore matters directly for ranking. This statement does not imply that arbitrary value distortions are harmless to NSGA-II: crowding distance uses numerical objective spacing. The practical teacher-selection criterion suggested by Theorem~\ref{thm:rank} is therefore to prefer a teacher whose induced \emph{signed bias is uniform across architecture families}, subject also to adequate predictive quality.
		\end{remark}
		
		The quantities in \eqref{eq:tau} must be estimated from information available before the main search if the bound is to be called a priori. In our protocol, $\rstr$ and $\sigma^2$ are estimated from pilot proxy evaluations, while $\bar\xi$ is estimated on a small pilot set spanning the architecture-capacity range. These pilot evaluations are reported as part of the evaluation budget and are not taken from the final search outcomes.
		
		\subsection{Fixed-Population Front and Hypervolume}
		
		Pairwise reliability lifts to the object NSGA-II actually selects. Since the first nondominated front changes only when a decisive dominance relation is reversed, the inversion probability of Theorem~\ref{thm:rank} bounds the chance that the proxy front differs from the true one.
		
		\begin{theorem}[Population-level first-front identification]
			\label{thm:front}
			Let $\mathcal{P}=\{a_1,\ldots,a_P\}$ be a \emph{fixed} candidate population. Let $F_1(\mathcal{P})$ be its first nondominated front under the true objectives and let $\widehat F_1(\mathcal{P})$ be the first front obtained when only the accuracy coordinate is replaced by the proxy score. Let $\mathcal{K}$ contain the candidate pairs whose dominance relation can change only through the accuracy coordinate while their feasibility status is unchanged; explicitly, $\mathcal{K} = \{(i,j): a_i \text{ and } a_j \text{ are incomparable or one dominates the other in all objectives except accuracy}\}$. Then
			\begin{equation}
				\Pr[\widehat F_1(\mathcal{P})\ne F_1(\mathcal{P})]
				\le
				\sum_{(i,j)\in\mathcal{K}}q(\Delta_{ij}),
				\label{eq:front}
			\end{equation}
			where $q(\cdot)$ is defined in \eqref{eq:rank}.
		\end{theorem}
		
		\begin{proof}
			For the fixed population, a difference between the true and proxy first fronts requires at least one critical dominance relation to be reversed. Hence
			\[
			\{\widehat F_1(\mathcal{P})\ne F_1(\mathcal{P})\}
			\subseteq
			\bigcup_{(i,j)\in\mathcal{K}}
			\{(a_i,a_j)\text{ is inverted}\}.
			\]
			Applying the union bound and Theorem~\ref{thm:rank} gives \eqref{eq:front}.
			\label{proof:front}
		\end{proof}
		
		We also state a separate value-perturbation result for hypervolume. Let
		$\mathbf{f}(a_i)$ be the true objective vector of candidate $a_i$, and let
		$\widehat{\mathbf{f}}(a_i)$ be the same vector with only the accuracy/error coordinate replaced by its proxy estimate. Denote these two sets of objective vectors for the same candidates by $\mathcal{Y}$ and $\widehat{\mathcal{Y}}$, with the same reference point used for both.
		
		\begin{proposition}[Hypervolume perturbation for the same candidate set]
			\label{prop:hv}
			Assume all objective vectors lie in the box
			$\prod_{j=1}^{d_{\mathrm{obj}}}[l_j,u_j]$, and let $V_{-e}=\prod_{j\ne e}(u_j-l_j)$ be the volume of the projection onto all coordinates except the estimated accuracy/error coordinate $e$. If
			$\epsilon_\infty=\max_{a\in\mathcal{P}}|S(a)-u(a)|$, then
			\begin{equation}
				|HV(\mathcal{Y})-HV(\widehat{\mathcal{Y}})|
				\le V_{-e}\epsilon_\infty.
				\label{eq:hvdet}
			\end{equation}
			Consequently, if $b_{\mathcal{P}}=\max_{a\in\mathcal{P}}|B(a)|$,
			\begin{equation}
				\E|HV(\mathcal{Y})-HV(\widehat{\mathcal{Y}})|
				\le
				V_{-e}\left[
				b_{\mathcal{P}}+
				\sEval\sqrt{2\log(2P)}
				\right].
				\label{eq:hvexp}
			\end{equation}
		\end{proposition}
		
		\begin{proof}
			Changing only one coordinate of every point by at most $\epsilon_\infty$ changes the dominated region only inside a slab of thickness at most $\epsilon_\infty$ along that coordinate. The slab's cross-sectional volume is at most $V_{-e}$, which gives \eqref{eq:hvdet}. From \eqref{eq:score-decomp},
			$\epsilon_\infty\le b_{\mathcal{P}}+\max_a|Z(a)|$.
			The standard maximal inequality for $P$ sub-Gaussian variables with uniform variance proxy $\sEval^2$ gives
			$\E[\max_a|Z(a)|]\le\sEval\sqrt{2\log(2P)}$; independence is not required for this upper bound.
		\end{proof}
		
		Theorem~\ref{thm:front} and Proposition~\ref{prop:hv} are population-level statements. They do not bound the final hypervolume difference between two complete evolutionary runs that visit different candidate sets. In the experiments, the relationship between ranking quality and end-of-run hypervolume is therefore tested empirically rather than presented as a theorem.
		
		\subsection{Derived Implementation Rules}
		
		The preceding analysis explains how the score behaves and provides two rules that can be used directly in the algorithm. The first sets the fusion weight between the GP and proxy from measured residual statistics. The second calibrates the capacity-adaptive distillation slope to cancel the component of bias aligned with model capacity.
		
		\begin{proposition}[Clipped variance-aware fusion]
			\label{prop:beta}
			Suppose the normalized GP score and normalized TGL score are both compared with the same fully evaluated screening target on the current archive. After centering the resulting residual sequences, let their residual standard deviations be $\sigma_{r,\mathrm{gp}}$ and $\sigma_{r,\mathrm{tgl}}$, with residual correlation $\varrho$. These residual quantities are distinct from the GP posterior uncertainty $\sigma_{\mathrm{gp}}(a)$ used in the exploration term. For the convex combination
			$\beta S_{\mathrm{gp}}+(1-\beta)S_{\mathrm{tgl}}$, $\beta\in[0,1]$, the variance-minimizing coefficient is
			\begin{equation}
				\beta^\star=
				\Pi_{[0,1]}
				\left(
				\frac{\sigma_{r,\mathrm{tgl}}^2-\varrho\sigma_{r,\mathrm{gp}}\sigma_{r,\mathrm{tgl}}}
				{\sigma_{r,\mathrm{gp}}^2+\sigma_{r,\mathrm{tgl}}^2
					-2\varrho\sigma_{r,\mathrm{gp}}\sigma_{r,\mathrm{tgl}}}
				\right),
				\label{eq:beta}
			\end{equation}
			where $\Pi_{[0,1]}(x)=\min\{1,\max\{0,x\}\}$ denotes projection onto $[0,1]$.
		\end{proposition}
		
		\begin{proof}
			The residual variance is a convex quadratic in $\beta$. Its unconstrained stationary point is the fraction inside the parentheses in \eqref{eq:beta}. Projection onto $[0,1]$ gives the constrained minimum. If the stationary point lies outside the interval, the optimum is one of the two endpoints.
		\end{proof}
		
		The residual correlation $\varrho$ is estimated from candidates in the fully evaluated archive $\mathcal{E}$ for which both normalized screening scores and the same fully evaluated reference target are available. Because this is a variance rule rather than a complete ranking-risk optimum, the coefficient is recalculated from the growing archive in every generation.
		
		\begin{proposition}[First-order removal of capacity-aligned signed bias]
			\label{prop:alpha}
			Let $\pi(a)\in[0,1]$ be the normalized architecture-capacity variable. Suppose that, over a pilot set spanning the capacity range, the signed bias in \eqref{eq:signed-bias} has the first-order form
			\begin{equation}
				B(a)
				=
				B_0+
				\left(\gamma_1+\kappa_T\alpha_1\right)\pi(a)
				+r(a),
				\label{eq:alpha-model}
			\end{equation}
			where $\gamma_1$ is the capacity-aligned slope without adaptive teacher weighting, $\kappa_T\ne0$ is the first-order change in that slope produced by the adaptive KD coefficient, and $r(a)$ contains higher-order residual terms. Then
			\begin{equation}
				\alpha_1^\star=-\gamma_1/\kappa_T
				\label{eq:alphastar}
			\end{equation}
			removes the linear capacity-aligned component when $\alpha_1^\star$ lies in the feasible interval imposed by $\alpha(a)\in[0,1]$.
		\end{proposition}
		
		\begin{proof}
			The coefficient multiplying $\pi(a)$ in \eqref{eq:alpha-model} is
			$\gamma_1+\kappa_T\alpha_1$. Setting it to zero gives \eqref{eq:alphastar}. If the solution is infeasible, the closest feasible endpoint minimizes the absolute first-order slope.
		\end{proof}
		
		Equation~\eqref{eq:alphastar} does not imply that $\alpha_1$ must always be positive. A positive adaptive slope is justified only when the pilot estimates show that $\gamma_1$ and $\kappa_T$ have opposite signs. In the implementation, $\alpha_1$ is calibrated on a small pilot set and then clipped so that $\alpha(a)\in[0,1]$ for every architecture.
		
		\subsection{Scope and Limitations}
		\label{sec:limits}
		
		The guarantees above have five explicit limits. First, Assumptions~\ref{as:linear} and \ref{as:drift} describe a short-training linearized regime rather than unrestricted deep-network optimization. Second, Assumption~\ref{as:bias} is needed to convert an optimization residual into an accuracy-bias bound; without a margin-transfer condition, classification accuracy is not a Lipschitz function of cross-entropy loss. Third, Lemma~\ref{lem:strat} applies to the independent evaluation set $\Dev$ after conditioning on the trained predictor; it does not by itself prove that stratified \emph{training} reduces model variance. Fourth, the ranking theorem permits dependence between candidate evaluation noises, but still assumes each conditional tail bound is valid. Fifth, Theorem~\ref{thm:front} and Proposition~\ref{prop:hv} concern a fixed candidate population; they do not capture error propagation through later variation operators or the value-sensitive crowding-distance calculation of NSGA-II. Accordingly, the experiments are designed to test the consequences of these guarantees rather than to verify the assumptions directly, and the results should be interpreted within these stated limits.
		\section{Experimental Setup}
		\label{sec:expdesign}
		
		The experiments follow the theoretical chain developed in Section~\ref{sec:theory}. A good final Pareto front is important, but it does not explain whether the low-fidelity evaluator worked as intended. We therefore begin by testing whether the proxy preserves the ordering of candidates. We then isolate the mechanisms responsible for that ordering before examining the complete constrained search. This sequence allows the final search quality to be interpreted in light of the proxy behaviour rather than from hypervolume alone.
		
		\subsection{Evaluation Plan}
		
		Table~\ref{tab:evidence} summarizes this progression from the low-fidelity estimator to the complete search. Repeated scoring of a fixed, fully evaluated architecture set first measures whether the proxy preserves candidate order. Component ablations then show how stratification, curriculum, and teacher quality affect that order. Equal-budget searches finally measure Pareto-front quality, constraint satisfaction, computational cost, and the properties of the selected architectures. These search-level results are empirical outcomes. They are not treated as a proof of Theorem~\ref{thm:front}, which applies only to a fixed candidate population.
		
		\begin{table}[!htbp]
			\centering
			\caption{Sequence of experimental analyses and the corresponding evidence.}
			\label{tab:evidence}
			\footnotesize
			\setlength{\tabcolsep}{2.5pt}
			\begin{tabularx}{\columnwidth}{@{}>{\raggedright\arraybackslash}p{2.25cm}>{\raggedright\arraybackslash}X@{}}
				\toprule
				\textbf{Analysis} & \textbf{Evidence} \\
				\midrule
				Ranking fidelity & Predicted and measured Kendall-$\tau$ on the same fixed architecture set. \\[2pt]
				Component contribution & Random, difficulty-only, no-curriculum, and no-teacher ablations with variance and Kendall-$\tau$. \\[2pt]
				Teacher robustness & Strong, uniformly weak, and selectively weak teachers with common bias, differential bias, and Kendall-$\tau$. \\[2pt]
				Constrained search & Hypervolume, generational distance, spread, feasible rate, FPR, and set coverage under an equal full-evaluation budget. \\[2pt]
				Efficiency and selection & Wall-clock cost and Tchebycheff-ranked deployable architectures. \\
				\bottomrule
			\end{tabularx}
		\end{table}
		
		\subsection{Benchmarks}
		
		To instantiate this evaluation sequence, we use two audio-classification tasks. Audio is well suited to the study because it combines an expensive accuracy objective, a model-size objective, and a false-positive-rate (FPR) constraint with clear operational meaning. The first task is keyword spotting (KWS) on Google Speech Commands v2 (GSC) \cite{warden2018speech}, a standard TinyML benchmark for small-footprint recognition \cite{garai2026kwsreview}. It is used for both fixed-set ranking analysis and constrained architecture search. The second task is bird-call classification on BirdCLEF~2021 \cite{kahl2021birdnet}, where compact detectors are commonly obtained through multi-objective search and distillation \cite{ghosh2026birdkd}. BirdCLEF has stronger background variation and places greater practical importance on FPR. A detector that frequently responds to noise is not useful for continuous monitoring, even when its classification accuracy is high. For both tasks, audio is converted to log-mel spectrograms and classified by a compact residual convolutional network. The architecture encoding is discrete, and reliable accuracy is available only after training, so both tasks are genuine expensive combinatorial NAS problems.
		
		\subsection{Search Space and Experimental Protocol}
		\label{sec:setup}
		
		With the benchmark tasks established, we next define a common search protocol. The architecture encoding controls the number of convolutional layers, filter count, kernel size, batch normalization, dropout, and number of fully connected layers in the residual convolutional backbone of Fig.~\ref{fig:arch}. Skip connections support the training of deeper variants. Global average pooling reduces the parameter count and also acts as a structural regularizer, both of which are useful under a TinyML memory budget. The highlighted quantities in the figure correspond directly to the discrete design variables in Section~\ref{sec:problem}. NSGA-II, surrogate-assisted NSGA-II (SA-NSGA-II), and TGL-NSGA-II use the same population size and variation operators. Multi-objective Bayesian optimization (MOBO) retains its own acquisition procedure. In all surrogate-assisted methods, the GP is fitted only to architectures that have been fully evaluated.
		
		\begin{figure}[!htbp]
			\centering
			\includegraphics[width=0.92\linewidth]{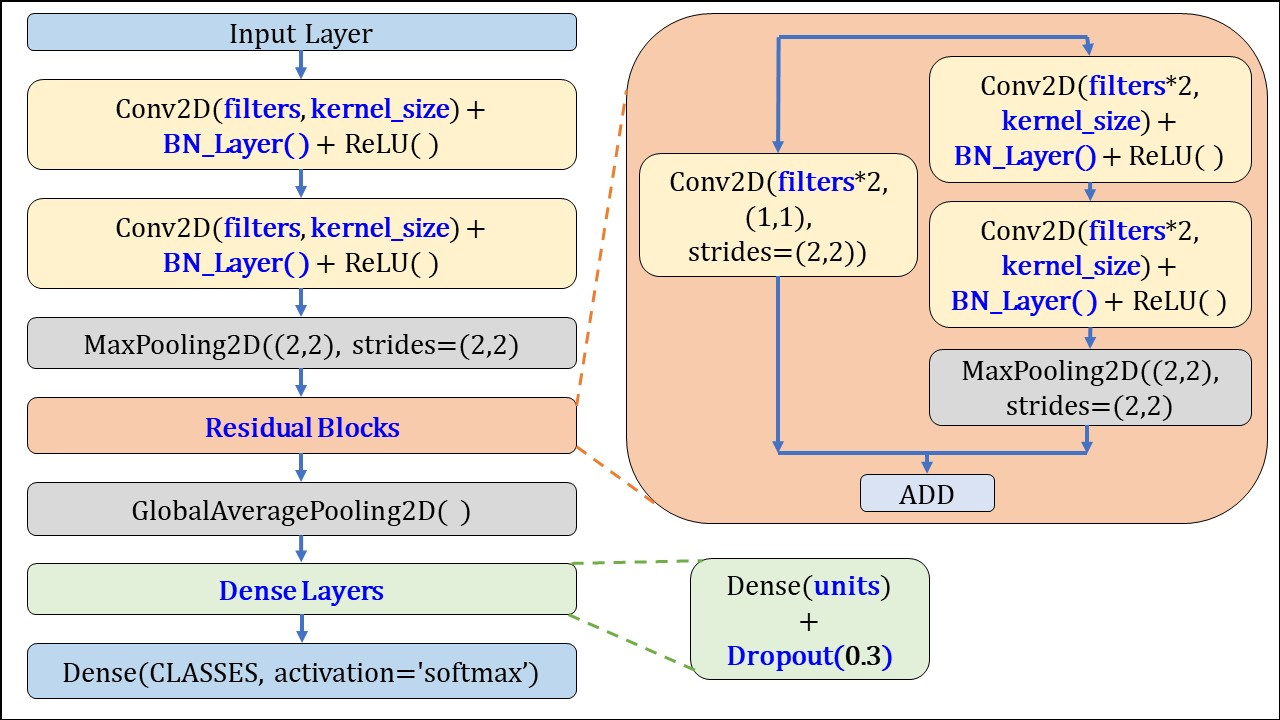}
			\caption{Residual convolutional backbone used for spectrogram-based audio classification. Bold blue tokens (\texttt{filters}, \texttt{kernel\_size}, batch-normalization flags, residual-block count, dense-layer \texttt{units}, and dropout) are the discrete design variables optimized by the search; a global average pooling layer precedes the softmax classifier to keep the parameter count within the TinyML budget.}
			\label{fig:arch}
		\end{figure}
		
		The default TGL compact training fraction is $\varphi=0.03$. The short curriculum runs three stages, each trained to convergence with early stopping and capped at $t_{\max}=15$ SGD steps, so the total proxy budget is $T_c\le45$ (typically much smaller). The proxy-training set $\Dtr$ and proxy-evaluation set $\Dev$ are disjoint, as required by Theorem~\ref{thm:conc}. For fair comparison, the primary budget is the number of expensive full-training evaluations. Wall-clock time is reported as a secondary measure because GP fitting and TGL scoring add different amounts of overhead.
		
		\emph{Scoring weights.} Section~\ref{sec:method} introduces several convex-combination weights: $\boldsymbol{\omega}$ for sample difficulty in \eqref{eq:difficulty}, $\boldsymbol{\psi}$ for fidelity in \eqref{eq:sfid}, $\rho$ for the balance between fidelity and performance in $S_{\mathrm{tgl}}$, $\nu$ for the balance between accuracy and FPR in \eqref{eq:sgp}, and $\eta_e$ for exploration in \eqref{eq:sfinal}. Two quantities are obtained from the analysis rather than tuned manually. Proposition~\ref{prop:beta} sets the fusion weight $\beta^\star$ in each generation, while Proposition~\ref{prop:alpha} determines the capacity-adaptive slope $\alpha_1$. For the remaining weights, we use equal values: $\boldsymbol{\omega}=\boldsymbol{\psi}=[\tfrac13,\tfrac13,\tfrac13]$ and $\rho=\nu=\eta_e=\tfrac12$. This neutral choice is reasonable because every component is normalized to $[0,1]$ and there is no prior reason to favour one cue. It also avoids task-specific tuning that could confound the ablations. A deployment with a stricter false-alarm requirement can instead increase $1-\nu$ or the FPR contribution controlled by $\rho$.
		
		\begin{table}[!htbp]
			\centering
			\caption{Main search configuration used for all reported runs. A dash (---) indicates that a parameter is not applicable to a method.}
			\label{tab:config}
			\scriptsize
			\setlength{\tabcolsep}{1.4pt}
			\begin{tabularx}{\columnwidth}{@{}>{\raggedright\arraybackslash}Xcccc@{}}
				\toprule
				\textbf{Parameter} & \shortstack{\textbf{NSGA-}\\\textbf{II}} & \textbf{MOBO} & \shortstack{\textbf{SA-}\\\textbf{NSGA-II}} & \textbf{TGL} \\
				\midrule
				Population $P$ & 40 & --- & 40 & 40 \\
				Generations $G$ & 50 & 50 & 50 & 50 \\
				Crossover / mutation & 0.9 / 0.1 & --- & 0.9 / 0.1 & 0.9 / 0.1 \\
				SBX / mutation index & 15 / 20 & --- & 15 / 20 & 15 / 20 \\
				Surrogate & --- & GP & GP & GP \\
				Promotion fraction $k$ & --- & --- & 20\% & 20\% \\
				KD temperature $T_{\mathrm{KD}}$ & --- & --- & --- & $2$ to $5$ \\
				$\alpha_0$ & --- & --- & --- & 0.3 \\
				Difficulty weights $\boldsymbol{\omega}$ & --- & --- & --- & $[\tfrac13,\tfrac13,\tfrac13]$ \\
				Fidelity weights $\boldsymbol{\psi}$ & --- & --- & --- & $[\tfrac13,\tfrac13,\tfrac13]$ \\
				Balances $\rho,\nu,\eta_e$ & --- & --- & --- & $\tfrac12$ each \\
				Fusion weight $\beta^\star$ & --- & --- & --- & Prop.~\ref{prop:beta} \\
				Adaptive slope $\alpha_1$ & --- & --- & --- & Prop.~\ref{prop:alpha} \\
				Compact fraction $\varphi$ & --- & --- & --- & 3\% \\
				Proxy steps per stage (cap $t_{\max}$) & --- & --- & --- & 15 \\
				Constraint penalty $\zeta_t$ & --- & --- & --- & increasing schedule \\
				\bottomrule
			\end{tabularx}
		\end{table}
		
		\subsection{Baselines and Ablations}
		
		Once the common protocol is fixed, the comparison methods can be used to separate the contributions of evolutionary search, surrogate assistance, and low-fidelity evaluation. Full NSGA-II provides the standard evolutionary baseline \cite{deb2002nsga2}. Multi-objective Bayesian optimization (MOBO) follows the TinyML formulation of Deutel \emph{et al.} \cite{deutel2023mobo}. SA-NSGA-II serves as the surrogate-only control. It uses the same NSGA-II population and variation operators as TGL-NSGA-II, but candidate promotion depends only on the GP surrogate. Its construction follows the surrogate-assisted comparison and rank-prediction literature \cite{tian2023pairwise,xue2024tcmr}. The random-subset proxy is a matched low-fidelity control in which candidates are trained and evaluated on randomly sampled compact data, representing reduced-training performance estimators \cite{kyriakides2020reduced}. Zero-cost NSGA-II uses a zero-cost NAS proxy in the style of Abdelfattah \emph{et al.} \cite{abdelfattah2021zerocost}. TGL-NSGA-II is the proposed method.
		
		The ablations then examine the proposed evaluator one component at a time. Joint stratification is first replaced with random sampling. Class refinement is removed in a second variant, leaving only difficulty strata. A third variant removes curriculum staging while retaining the same number of proxy steps, and a fourth removes teacher guidance. Separate teacher-quality controls compare a strong teacher with uniformly and selectively weakened teachers while keeping their average quality as close as possible.
		
		\subsection{Evaluation Metrics}
		
		These comparisons require metrics at both the ranking and search levels. Ranking quality is measured by Kendall-$\tau$, Spearman rank correlation $\rho_s$, and Top-5 overlap. Kendall-$\tau$ is the primary measure because it appears directly in Theorem~\ref{thm:rank}. Top-5 overlap shows whether the low-fidelity evaluator identifies the same strongest candidates as full evaluation. For the diagnostic across generations, Spearman correlation is mapped from $[-1,1]$ to $[0,1]$ and then averaged with Top-5 overlap. This combined score complements the fixed-set Kendall-$\tau$ analysis rather than replacing it. The Top-5 models are defined after applying the same Tchebycheff scalarization to the proxy and fully evaluated objective vectors, which makes the two ranking views directly comparable in each generation.
		
		Search quality is measured using hypervolume (HV), generational distance (GD), spread, and feasible rate (FR). Higher values are preferred for HV and FR, while lower values are preferred for GD and spread. FPR is reported separately because it is an operational constraint in the audio tasks. These search-level indicators are interpreted together with ranking fidelity. This prevents a favourable final front from being attributed to the low-fidelity mechanism without direct evidence that candidate order was preserved.
		
		
		\subsection{Reporting Protocol}
		\label{sec:statistics}
		
		The reporting protocol follows naturally from these two levels of evaluation. Unless a figure is identified as a representative single-run diagnostic, search metrics are reported as mean\,$\pm$\,standard deviation over five independent random seeds. The same seeds are used across methods to keep the comparison controlled. The fixed-set tables instead report aggregate estimates from repeated proxy evaluations. In those tables, measured Kendall-$\tau$ is compared directly with its predicted lower bound, while the ablations are interpreted through evaluation variance, differential bias, and Kendall-$\tau$. We describe these patterns as empirical differences and do not claim statistical significance without a corresponding inferential test.

		\section{Experimental Results}
		\label{sec:results}
		
		Following the experimental design, we present the results in four stages: fixed-set ranking fidelity, component ablations, complete constrained search, and computational cost. This order separates the quality of the low-fidelity evaluator from the final search outcome and helps explain why the method succeeds or fails.
		
		\subsection{Proxy Ranking Fidelity}
		
		A fixed architecture set is sampled before the search and fully trained to obtain the reference order $u(a)$. The same candidates are then scored repeatedly by TGL using independent draws of $\Dev$. The quantities $\rstr$, $\sigma^2$, and $\bar\xi$ are estimated using the pilot procedure described after Theorem~\ref{thm:rank}, and Eq.~\eqref{eq:tau} gives the corresponding lower bound on Kendall-$\tau$.
		
		\begin{table}[!htbp]
			\centering
			\caption{Predicted and measured ranking fidelity on the fixed architecture set. Values are aggregate estimates from repeated proxy evaluations.}
			\label{tab:rank-fidelity}
			\footnotesize
			\setlength{\tabcolsep}{2.8pt}
			\begin{tabular}{@{}lccccc@{}}
				\toprule
				\textbf{Dataset} & $\hat\rstr$ & $\hat\sigma^2$ & $\hat{\bar\xi}$ & $\tau_{\mathrm{pred}}$ & $\tau_{\mathrm{meas}}$ \\
				\midrule
				GSC & 0.42 & 0.0964 & 0.018 & 0.60 & 0.74 \\
				BirdCLEF & 0.34 & 0.1642 & 0.028 & 0.46 & 0.62 \\
				\bottomrule
			\end{tabular}
		\end{table}
		
		Table~\ref{tab:rank-fidelity} shows that the measured ranking fidelity exceeds the theoretical lower bound on both tasks. For GSC, the predicted Kendall-$\tau$ is $0.60$, whereas the measured value is $0.74$, a margin of $0.14$. BirdCLEF is more difficult: its lower stratification gain ($\hat\rstr=0.34$ versus $0.42$), larger score variance ($0.1642$ versus $0.0964$), and larger differential bias ($0.028$ versus $0.018$) reduce the predicted and measured Kendall-$\tau$ to $0.46$ and $0.62$, respectively. Nevertheless, the measured value remains $0.16$ above the predicted bound. The cross-dataset pattern is therefore intuitive: the noisier and more heterogeneous BirdCLEF task produces weaker candidate ordering, but the bound remains conservative in both cases.
		
		Figure~\ref{fig:surrogate-fidelity} provides a complementary view across generations. It compares SA-NSGA-II with the teacher-guided variant used in the audio runs. The figure retains the earlier label SA-KD-NSGA-II for this variant, which is now incorporated into TGL-NSGA-II. The teacher-guided score is strongest in the early generations, when the GP archive contains relatively few observations. The gap narrows as more architectures are fully evaluated and the surrogate improves. This behaviour is consistent with the intended role of TGL during surrogate cold start. The figure uses normalized Spearman correlation and Top-5 overlap rather than Kendall-$\tau$, so it should be read as supporting evidence rather than a direct test of Theorem~\ref{thm:rank}.
		
		\begin{figure}[!htbp]
			\centering
			\includegraphics[width=\linewidth]{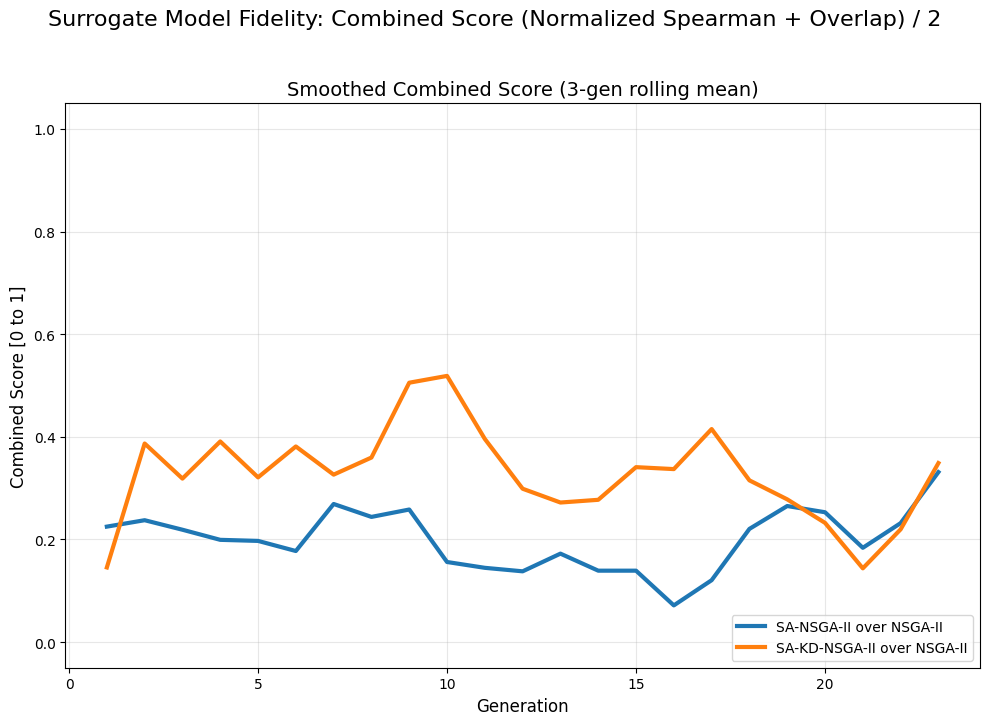}
			\caption{Generation-wise surrogate fidelity (average of normalized Spearman correlation and Top-5 overlap, three-generation rolling mean). The legacy label SA-KD-NSGA-II denotes the teacher-guided variant folded into TGL-NSGA-II.}
			\label{fig:surrogate-fidelity}
		\end{figure}
		
		The ranking results show that the low-fidelity evaluator preserves order. The next set of experiments explains which components produce this behaviour.
		
		\subsection{Component and Teacher Ablations}
		
		Table~\ref{tab:mechanism} isolates the contribution of each part of the evaluator on the same fixed architecture set. The first group compares random evaluation, difficulty-only strata, strata defined jointly by difficulty and class, removal of the curriculum, and removal of teacher guidance. The second group distinguishes uniformly weak supervision from selective teacher mismatch. This distinction matters because a common shift in candidate scores need not change their order, whereas architecture-dependent error directly changes pairwise differences.
		
		\begin{table}[!htbp]
			\centering
			\caption{Component and teacher-quality ablations on the fixed architecture set. Variance is normalized by random evaluation; lower bias terms and higher Kendall-$\tau$ are better. N/A indicates that a quantity is not applicable to that control.}
			\label{tab:mechanism}
			\footnotesize
			\setlength{\tabcolsep}{3pt}
			\begin{tabular}{@{}lcccc@{}}
				\toprule
				\textbf{Variant / Teacher} & $\Var(S)/\Var(S_{\mathrm{rand}})$ & $|\hat{\bar B}|$ & $\hat{\bar\xi}$ & Kendall-$\tau$ \\
				\midrule
				Full TGL (strong teacher) & 0.59 & 0.042 & 0.021 & 0.68 \\
				Random evaluation         & 1.00 & N/A & N/A & 0.56 \\
				Difficulty strata only    & 0.71 & N/A & N/A & 0.63 \\
				Without curriculum        & 0.60 & N/A & N/A & 0.64 \\
				Without teacher guidance  & 0.83 & N/A & N/A & 0.49 \\
				Uniformly weak teacher   & N/A & 0.086 & 0.025 & 0.62 \\
				Selectively weak teacher & N/A & 0.082 & 0.061 & 0.41 \\
				\bottomrule
			\end{tabular}
		\end{table}
		
		Joint stratification reduces the normalized score variance from $1.00$ under random evaluation to $0.59$, corresponding to a $41\%$ reduction, and raises Kendall-$\tau$ from $0.56$ to $0.68$. Using difficulty strata without class refinement still helps, but its variance ratio of $0.71$ and Kendall-$\tau$ of $0.63$ are weaker than those of the full method. This shows that difficulty balancing provides most of the sampling benefit, while class refinement supplies an additional gain by preventing common classes from dominating individual strata.
		
		Removing curriculum produces a variance ratio of $0.60$, almost identical to the full method, but lowers Kendall-$\tau$ from $0.68$ to $0.64$. This is consistent with the design: curriculum changes the short training trajectory, not the variance of the independent evaluation sample. Removing teacher guidance has a much larger effect. Its variance ratio rises to $0.83$ and Kendall-$\tau$ falls to $0.49$, indicating that the teacher contributes both to more homogeneous sampling strata and to better short-training comparisons.
		
		The teacher-quality controls clarify that average bias alone is not the decisive quantity. Weakening the teacher uniformly approximately doubles the common bias from $0.042$ to $0.086$. Even so, differential bias remains small at $0.025$, and Kendall-$\tau$ decreases only moderately from $0.68$ to $0.62$. The selectively weak teacher has a similar common-bias magnitude of $0.082$, but its differential bias increases to $0.061$ and Kendall-$\tau$ falls to $0.41$. A teacher can therefore be imperfect in a roughly uniform way without destroying the candidate order. Selective mismatch is more harmful because it favours some architecture families over others.
		
		Together, these ablations explain where the ranking improvement comes from. We next examine whether the same improvement is visible in the complete constrained search.
		
		\subsection{Constrained Search Quality}
		
		The representative search traces provide the first view of end-to-end behaviour. In the KWS hypervolume trace of Fig.~\ref{fig:hv-kws}, TGL-NSGA-II improves steadily and finishes above full NSGA-II \cite{deb2002nsga2} and the surrogate-only run. MOBO \cite{deutel2023mobo} makes a sharp early jump and reaches a slightly larger HV in this particular run. The main feature of the TGL trajectory is therefore its steady improvement under a restricted evaluation budget, rather than dominance at every generation.
		
		\begin{figure}[!htbp]
			\centering
			\subfloat[KWS hypervolume.\label{fig:hv-kws}]{%
				\includegraphics[width=0.49\linewidth]{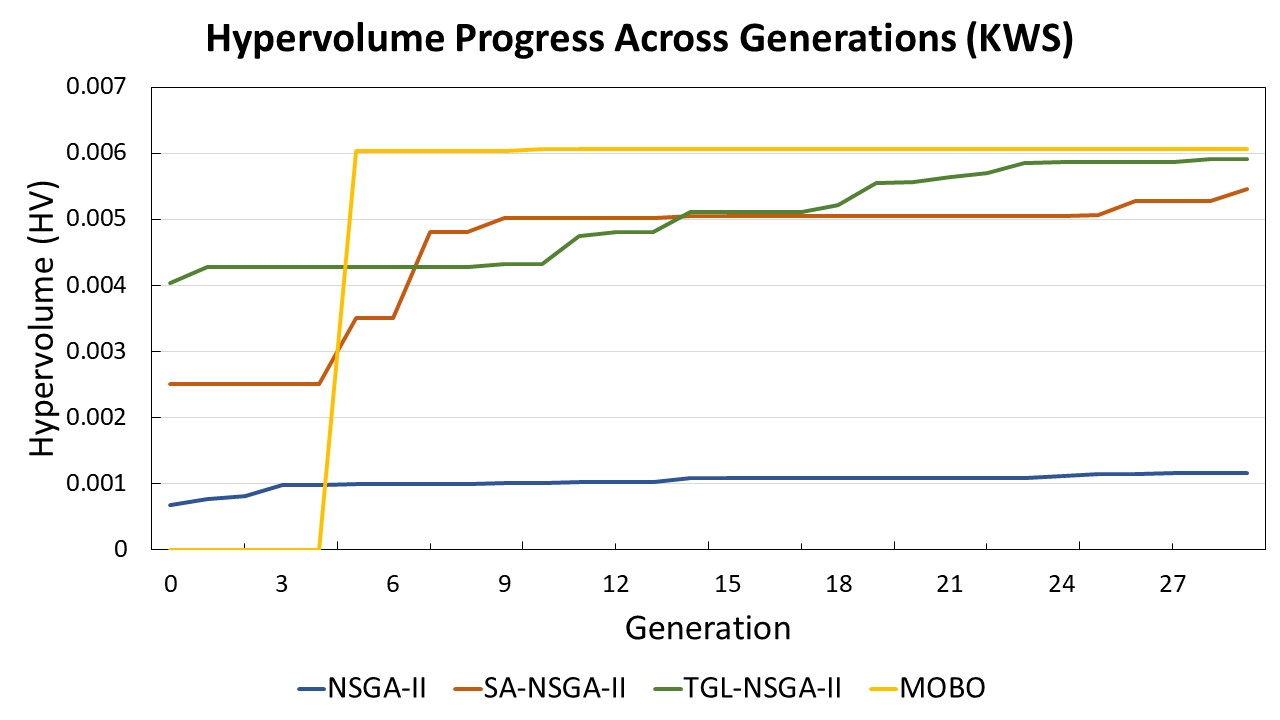}}
			\hfill
			\subfloat[Bird-call feasible rate.\label{fig:fr-bird}]{%
				\includegraphics[width=0.49\linewidth]{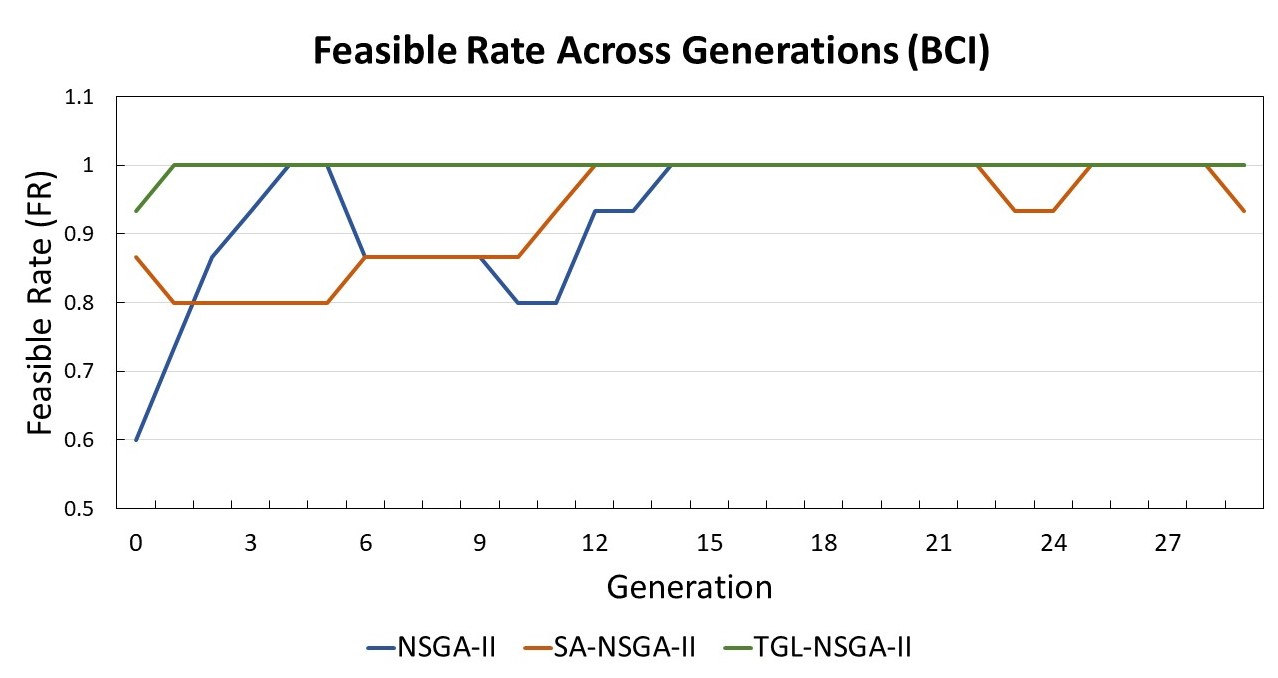}}
			\caption{Representative single-run audio search diagnostics. (a) HV rises steadily for TGL-NSGA-II; (b) TGL reaches a fully feasible population early and stays near one.}
			\label{fig:hv-fr}
		\end{figure}
		
		Figure~\ref{fig:fpr-bird} provides the corresponding operational view on BirdCLEF. TGL-NSGA-II has the lowest mean FPR among the four methods. This matters in continuous acoustic monitoring because frequent false alarms can make a model unusable even when its accuracy is high. TGL also shows the broadest balanced profile after model size and FPR are inverted so that outward values are better. NSGA-II performs well on inverted model size but is weaker in accuracy and FPR, which reinforces the need to report all objectives rather than accuracy alone.
		
		\begin{figure}[!htbp]
			\centering
			\subfloat[Mean FPR $\pm$ standard deviation.]{%
				\includegraphics[width=0.52\linewidth]{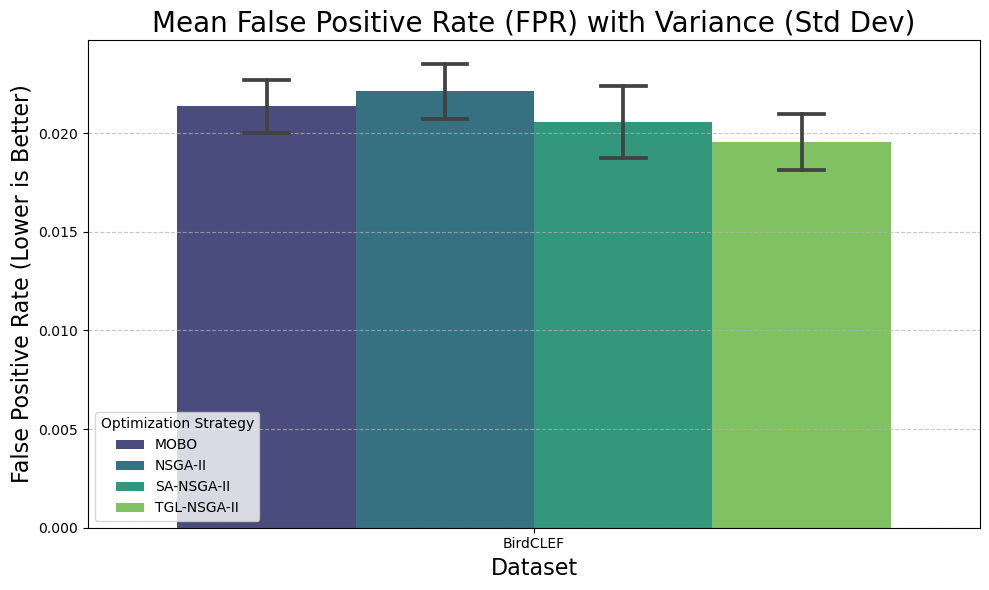}}
			\hfill
			\subfloat[Normalized profile.]{%
				\includegraphics[width=0.46\linewidth]{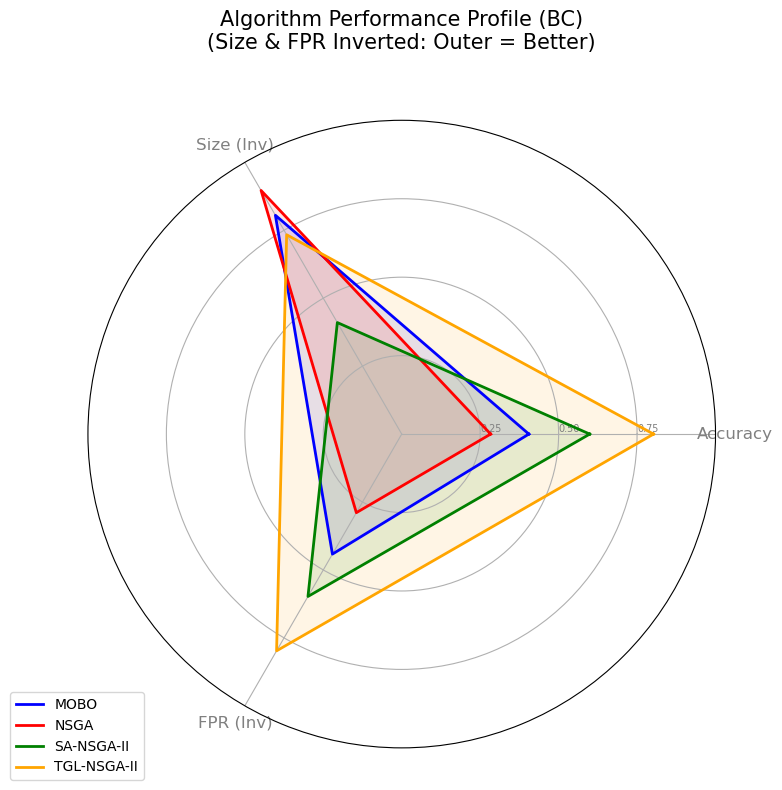}}
			\caption{BirdCLEF operational analysis. (a) TGL-NSGA-II attains the lowest mean FPR (lower is better). (b) Three-objective profile with model size and FPR inverted so larger radial values are better on all axes.}
			\label{fig:fpr-bird}
		\end{figure}
		
		The five-seed results in Table~\ref{tab:front-quality} provide a more stable comparison through hypervolume (HV), generational distance (GD), and spread. On KWS, TGL-NSGA-II attains the largest mean HV and the smallest mean GD. It therefore covers a favourable portion of the objective space while remaining closest to the reference front. Its spread is comparable to that of NSGA-II, although SA-NSGA-II produces the most uniform front. MOBO achieves a strong mean HV but has much larger GD and the worst spread among the conventional baselines. The zero-cost proxy also improves HV over plain NSGA-II, but its very large GD shows that many selected candidates remain far from the reference front. TGL is therefore strongest when coverage and convergence are considered together, even though it is not best on every individual indicator.
		
		\begin{table}[!htbp]
			\centering
			\caption{Pareto-front quality on the KWS runs (mean\,$\pm$\,standard deviation over five seeds; higher HV is better, lower GD and spread are better). Best mean per column in bold.}
			\label{tab:front-quality}
			\scriptsize
			\setlength{\tabcolsep}{2pt}
			\resizebox{\columnwidth}{!}{%
				\begin{tabular}{@{}lccc@{}}
					\toprule
					\textbf{Method} & \textbf{HV}$\uparrow$ & \textbf{GD}$\downarrow$ & \textbf{Spread}$\downarrow$ \\
					\midrule
					NSGA-II \cite{deb2002nsga2}            & $0.0049{\pm}0.0006$ & $0.0032{\pm}0.0004$ & $1.028{\pm}0.031$ \\
					MOBO \cite{deutel2023mobo}             & $0.0809{\pm}0.0071$ & $0.1210{\pm}0.0153$ & $1.218{\pm}0.048$ \\
					Zero-Cost NSGA-II \cite{abdelfattah2021zerocost} & $0.0714{\pm}0.0090$ & $6.290{\pm}0.812$ & $1.921{\pm}0.087$ \\
					SA-NSGA-II \cite{tian2023pairwise}     & $0.0878{\pm}0.0043$ & $0.0459{\pm}0.0061$ & $\mathbf{0.961{\pm}0.024}$ \\
					\textbf{TGL-NSGA-II}                   & $\mathbf{0.0935{\pm}0.0038}$ & $\mathbf{0.0019{\pm}0.0003}$ & $1.049{\pm}0.029$ \\
					\bottomrule
			\end{tabular}}
		\end{table}
		
		The pairwise set-coverage ($C$-metric) in Table~\ref{tab:cmat} provides a solution-level view of these fronts. TGL-NSGA-II dominates $100\%$ of the zero-cost front, $80\%$ of the SA-NSGA-II front, $50\%$ of the MOBO front, and $7.1\%$ of the NSGA-II front. In the reverse direction, no solution from NSGA-II, MOBO, or the zero-cost method dominates a TGL solution; SA-NSGA-II dominates $15.4\%$ of the TGL front. This asymmetry supports the favourable coverage of TGL while also showing that SA-NSGA-II retains a small complementary region of the objective space.
		
		\begin{table}[!htbp]
			\centering
			\caption{Set-coverage $C$-matrix on the KWS runs (fronts pooled over five seeds). Entry $(i,j)$ is the fraction of row-$i$ solutions dominated by at least one column-$j$ solution; diagonal entries are not applicable.}
			\label{tab:cmat}
			\scalebox{0.72}{
				\begin{tabular}{c|c|c|c|c|c}
					\hline
					& \textbf{TGL} & \textbf{NSGA-II} & \textbf{MOBO} & \textbf{SA-NSGA-II} & \textbf{Zero-Cost} \\ \hline\hline
					\textbf{TGL-NSGA-II} & N/A & 0.000 & 0.000 & 0.154 & 0.000 \\
					\textbf{NSGA-II}     & 0.071 & N/A & 0.071 & 0.000 & 0.000 \\
					\textbf{MOBO}        & 0.500 & 0.000 & N/A & 0.250 & 0.000 \\
					\textbf{SA-NSGA-II}  & 0.800 & 0.000 & 0.000 & N/A & 0.000 \\
					\textbf{Zero-Cost}   & 1.000 & 0.000 & 0.000 & 0.667 & N/A \\ \hline
				\end{tabular}
			}
		\end{table}
		
		The HV and feasible-rate traces in Fig.~\ref{fig:hv-fr} show that these advantages appear early and persist. TGL-NSGA-II rises quickly and remains stable, while the baselines fluctuate more strongly. Together with Tables~\ref{tab:rank-fidelity} and~\ref{tab:mechanism}, these search results show that better ranking fidelity is accompanied by better feasible-front quality. This association supports the practical value of the low-fidelity evaluator. It does not, however, turn Theorem~\ref{thm:front} into a convergence guarantee for the complete evolutionary trajectory.
		
		Search quality must also be weighed against the additional cost of KD-Lite. The final results subsection examines this trade-off and the models selected from the returned fronts.
		
		\subsection{Runtime and Selected Architectures}
		
		Table~\ref{tab:cost} reports wall-clock cost on an NVIDIA RTX~A4000 GPU with an Intel Xeon W-2245 CPU. Values are shown as mean\,$\pm$\,standard deviation over five seeds. Full NSGA-II is the slowest method because every candidate is fully trained. TGL-NSGA-II adds KD-Lite to surrogate filtering, so it is slower than SA-NSGA-II and MOBO. Even so, it reduces runtime from $1389.58$~s to $621.71$~s relative to full NSGA-II, which corresponds to a $2.2\times$ speed-up. The proxy remains much cheaper than full training because $\mathcal{D}_{\mathrm{tr}}$ contains only $3\%$ of the data pool and each curriculum stage has a fixed step cap. The value of this additional cost lies in the stronger ranking fidelity and the favourable combination of HV and GD in Table~\ref{tab:front-quality}, not in runtime alone.
		
		\begin{table}[!htbp]
			\centering
			\caption{Measured computational cost on the audio search (seconds; mean\,$\pm$\,standard deviation over five seeds).}
			\label{tab:cost}
			\footnotesize
			\begin{tabular}{@{}lc@{}}
				\toprule
				\textbf{Method} & \textbf{Cost (s)}$\downarrow$ \\
				\midrule
				NSGA-II \cite{deb2002nsga2}       & $1389.58{\pm}42.7$ \\
				MOBO \cite{deutel2023mobo}         & $147.10{\pm}9.8$ \\
				SA-NSGA-II \cite{tian2023pairwise} & $224.39{\pm}12.3$ \\
				\textbf{TGL-NSGA-II}               & $621.71{\pm}18.5$ \\
				\bottomrule
			\end{tabular}
		\end{table}
		
		Runtime alone does not determine which architecture should be deployed. We therefore rank the returned models by Tchebycheff scalarization over negated accuracy, model size, and FPR. Each objective is normalized using its minimum and maximum, and user preferences are represented by $\boldsymbol{\lambda}$. The selected model minimizes $T(\theta)=\max_k \lambda_k\hat f_k(\theta)$, which avoids choosing a favourable point by hand after the search. Table~\ref{tab:rank} lists the leading configurations for both tasks. With the reported balanced weights, \texttt{GSC\_MOBO-1} is the top KWS configuration because it combines a very low FPR with a compact model and only a modest accuracy loss. On BirdCLEF, \texttt{BC\_SA-NSGA-II-1} has the lowest overall Tchebycheff score. The TGL alternatives occupy a different operating region. In particular, \texttt{BC\_TGL-NSGA-II-3} has the highest accuracy and lowest FPR in the table, but its larger size increases its balanced score. It becomes attractive when recognition quality and false-alarm control are given more importance than memory. The table therefore represents task-dependent trade-offs rather than a single optimizer that is best for every deployment preference.
		
		\begin{table*}[!htbp]
			\centering
			\caption{Tchebycheff-ranked model configurations for both audio tasks (accuracy, size, and FPR reported as mean\,$\pm$\,standard deviation over five seeds). Hyperparameters are listed as \{filters, kernel size, use BN, residual blocks, FC layers, use dropout\}. Lower score is better.}
			\label{tab:rank}
			\footnotesize
			\begin{tabular}{c|c|c|c|c|c|c}
				\hline
				\textbf{Dataset} & \textbf{Model} & \textbf{Accuracy} & \textbf{Size (MB)} & \textbf{FPR} & \textbf{Hyperparameters} & \textbf{Score} \\ \hline\hline
				\multirow{6}{*}{\textbf{GSC}}
				& GSC\_MOBO-1        & $0.9121{\pm}0.0043$ & $0.2743{\pm}0.006$ & $0.0106{\pm}0.0012$ & \{32,3,0,1,1,1\} & \textbf{0.0006} \\ \cline{2-7}
				& GSC\_MOBO-2        & $0.9024{\pm}0.0051$ & $0.2743{\pm}0.006$ & $0.0101{\pm}0.0011$ & \{32,2,0,1,1,1\} & 0.0034 \\ \cline{2-7}
				& GSC\_SA-NSGA-II-1  & $0.9059{\pm}0.0047$ & $0.3236{\pm}0.008$ & $0.0105{\pm}0.0013$ & \{16,5,1,2,2,1\} & 0.0041 \\ \cline{2-7}
				& GSC\_NSGA-II-1     & $0.9093{\pm}0.0049$ & $0.1854{\pm}0.005$ & $0.0896{\pm}0.0071$ & \{16,5,0,1,1,1\} & 0.0064 \\ \cline{2-7}
				& GSC\_TGL-NSGA-II-1 & $0.8841{\pm}0.0055$ & $0.1191{\pm}0.004$ & $0.0129{\pm}0.0015$ & \{16,3,1,2,1,0\} & 0.0126 \\ \cline{2-7}
				& GSC\_TGL-NSGA-II-2 & $0.8797{\pm}0.0058$ & $0.0982{\pm}0.003$ & $0.0134{\pm}0.0016$ & \{32,3,0,1,1,0\} & 0.0141 \\ \hline\hline
				\multirow{6}{*}{\textbf{BirdCLEF'21}}
				& BC\_SA-NSGA-II-1   & $0.7752{\pm}0.0062$ & $0.1650{\pm}0.005$ & $0.0226{\pm}0.0019$ & \{16,3,0,2,2,0\} & \textbf{0.0219} \\ \cline{2-7}
				& BC\_MOBO-1         & $0.7821{\pm}0.0059$ & $0.2252{\pm}0.006$ & $0.0218{\pm}0.0018$ & \{16,5,0,1,2,1\} & 0.0229 \\ \cline{2-7}
				& BC\_NSGA-II-1      & $0.7860{\pm}0.0057$ & $0.1857{\pm}0.005$ & $0.0208{\pm}0.0018$ & \{16,5,0,1,1,0\} & 0.0236 \\ \cline{2-7}
				& BC\_TGL-NSGA-II-1  & $0.7897{\pm}0.0054$ & $0.1650{\pm}0.005$ & $0.0212{\pm}0.0017$ & \{16,3,0,2,2,0\} & 0.0243 \\ \cline{2-7}
				& BC\_TGL-NSGA-II-2  & $0.8124{\pm}0.0050$ & $0.3222{\pm}0.008$ & $0.0189{\pm}0.0016$ & \{16,5,0,2,2,0\} & 0.0524 \\ \cline{2-7}
				& BC\_TGL-NSGA-II-3  & $0.8154{\pm}0.0048$ & $0.3239{\pm}0.008$ & $0.0185{\pm}0.0015$ & \{16,3,1,2,3,1\} & 0.0530 \\ \hline\hline
			\end{tabular}
		\end{table*}
		
		\section{Discussion}
		\label{sec:discussion}
		
		Taken together, the results tell a consistent story. GSC has a larger stratification gain, lower score variance, and smaller differential bias than BirdCLEF. It also achieves the higher measured Kendall-$\tau$, with $0.74$ compared with $0.62$. Both measured values remain above their predicted lower bounds of $0.60$ and $0.46$. The bound is therefore conservative, yet it still reflects the greater difficulty of BirdCLEF. This gives more insight than rank correlation alone because the observed ordering can be connected to specific sources of uncertainty.
		
		The ablations then explain how this ranking behaviour is produced. Joint stratification reduces score variance by $41\%$ relative to random evaluation. Class refinement adds a smaller but consistent improvement over difficulty-only strata. Curriculum has little effect on evaluation variance, but it improves the candidate order. Teacher guidance is the most influential component because removing it produces the lowest Kendall-$\tau$ among the structural ablations. The two weak-teacher controls reveal an equally important point. A uniformly weak teacher has about twice the common bias of the strong teacher, yet it preserves most of the ranking because its differential bias remains small. A selectively weak teacher has almost the same common bias, but its larger differential bias reduces Kendall-$\tau$ to $0.41$. In practice, teacher suitability should therefore be judged across architecture families rather than by average teacher accuracy alone.
		
		The fidelity curve in Fig.~\ref{fig:surrogate-fidelity} shows when the teacher-guided signal is most useful. Early in the search, the GP archive contains few fully evaluated architectures, so TGL provides the stronger guidance. The gap narrows as the archive grows and the GP improves. On KWS, this behaviour is accompanied by the best mean HV and GD over five seeds. SA-NSGA-II still provides the best spread, and MOBO briefly leads in the representative single-run trace. TGL should therefore be viewed as improving the reliability and convergence of the front, not as guaranteeing the best value in every generation or for every metric. BirdCLEF provides the stronger evidence for constraint handling. TGL reaches a feasible rate close to one and records the lowest mean FPR, which is particularly important for continuous acoustic monitoring.
		
		These improvements come with a clear computational trade-off. TGL is slower than MOBO and surrogate-only SA-NSGA-II because every candidate receives a KD-Lite evaluation. It is nevertheless $2.2\times$ faster than full NSGA-II. The additional computation is accompanied by better fixed-set ranking and a stronger combination of HV and GD. Table~\ref{tab:rank} also shows why search performance and deployment choice should be treated separately. Balanced scalarization favours MOBO on GSC and SA-NSGA-II on BirdCLEF. The most accurate BirdCLEF model with the lowest FPR comes from TGL, but it also requires more memory.
		
		This interpretation identifies both the strengths of the method and the conditions under which they hold. The remaining limitations are discussed next.
		
		\subsection{Limitations}
		
		The theoretical guarantees apply to a fixed candidate population and should not be read as an end-to-end convergence theorem. NSGA-II crowding distance depends on numerical spacing, so the ranking theory applies more directly to dominance than to value-sensitive diversity selection. The short-training analysis also relies on the linearized regime and the margin-transfer condition in Assumption~\ref{as:bias}. These are modelling assumptions rather than universal properties of deep networks.
		
		The method also requires a pretrained teacher and a small pilot set. The selectively weak teacher experiment shows that architecture-dependent mismatch remains a genuine failure mode. In addition, the empirical study covers only two audio tasks, so transfer to other modalities remains to be established. The plots across generations are representative single-run diagnostics and should be interpreted together with the five-seed tables rather than as stand-alone statistical evidence. Finally, wall-clock time and model size cannot replace direct measurement of on-device latency and energy. The scope of the paper is therefore a population-level ranking principle demonstrated within a constrained evolutionary search. Improvements over the complete search are empirical outcomes rather than theoretical guarantees.
		
		\section{Conclusion}
		
		This work began with a simple observation: an expensive evolutionary search often needs reliable comparisons more than exact low-fidelity fitness values. TGL-NSGA-II addresses this need by combining KD-Lite, a short teacher-guided procedure for low-fidelity fitness estimation, with an independent stratified evaluation set and a GP surrogate. The analysis separates systematic bias from sampling noise and connects both quantities to pairwise inversion, expected Kendall-$\tau$, and first-front identification for a fixed population.
		
		The experimental results follow the same chain. On GSC and BirdCLEF, the measured Kendall-$\tau$ values of $0.74$ and $0.62$ exceed the predicted lower bounds of $0.60$ and $0.46$. Joint stratification reduces score variance by $41\%$. The teacher controls further show that differential bias is more damaging to candidate ordering than a similar amount of common bias, with selective mismatch reducing Kendall-$\tau$ to $0.41$. In the complete search, TGL provides useful guidance during surrogate cold start, reaches a feasible rate close to one and the lowest mean FPR on BirdCLEF, and obtains the largest mean HV and smallest mean GD on KWS. It is also $2.2\times$ faster than full NSGA-II, although MOBO and SA-NSGA-II remain faster and may provide preferable individual architectures under particular deployment weights.
		
		The conclusion is therefore specific rather than universal. When sampling variance and architecture-dependent bias are controlled, rank-aware low-fidelity evaluation can make expensive constrained evolutionary search more reliable. It improves the quality of candidate comparisons and can translate that improvement into a stronger feasible front without requiring every candidate to be fully trained.

		\bibliographystyle{IEEEtran}
		\bibliography{refs_tgl_verified_new}
	}
\end{document}